\documentclass[12pt,a4paper]{article}

\usepackage[margin=1in]{geometry} %
\usepackage[utf8]{inputenc} %
\usepackage{authblk} %
\usepackage[colorlinks=True, citecolor=red, linkcolor=blue]{hyperref} %
\usepackage[numbers,sort]{natbib} %
\usepackage{color}
\usepackage{array,multirow}

\usepackage{enumitem}
\setlist{itemsep=0pt}

\usepackage{amssymb,amsmath,amsthm}
\usepackage{mathtools}
\usepackage{graphicx}
\usepackage{subcaption}
\usepackage[ruled,linesnumbered]{algorithm2e}
\usepackage[capitalise]{cleveref}



\theoremstyle{definition}
\newtheorem{theorem}{Theorem}[section]
\newtheorem{proposition}{Proposition}

\theoremstyle{remark}
\newtheorem{remark}[theorem]{Remark}

\newcommand{\R}{\mathbb{R}}
\newcommand{\E}{\mathbb{E}}

\DeclareMathOperator{\trace}{tr}

\DeclareMathOperator{\Var}{Var}
\DeclareMathOperator{\Cov}{Cov}

\renewcommand{\d}{\mathrm{d}}

\usepackage{xargs}                      %
\usepackage[pdftex,dvipsnames]{xcolor}  %
\usepackage[colorinlistoftodos,prependcaption,textsize=tiny,textwidth=0.85in]{todonotes}

\usepackage[normalem]{ulem} %

\begin{document}

\title{Gradient-free Riemannian Langevin Sampler}

\author{
Ricardo Baptista\footnote{Department of Statistical Sciences, University of Toronto, Canada, \url{r.baptista@utoronto.ca}},
Olivier Zahm\footnote{UGA, Inria, CNRS, Grenoble INP*, LJK, 38000 Grenoble, France, \url{olivier.zahm@inria.fr}}
}

\maketitle

\begin{abstract}

We address the problem of efficiently sampling multimodal probability distributions, where standard Markov Chain Monte Carlo methods often suffer from poor mixing and mode trapping. To mitigate these issues, we propose Gradient-free Riemannian Langevin Sampler (GRiLS), a novel proposal that improves exploration without requiring gradient evaluations of the target density.
Our approach introduces a Riemannian metric which reshapes the local geometry in order to facilitate transitions across modes.
The resulting gradient-free MCMC algorithm is particularly suitable for complex, computationally expensive targets where derivatives are unavailable or impractical.
The GRiLS proposal requires knowing the mean and covariance of the target density, which we estimate using an ensemble of interacting particles.
Empirical results on multimodal benchmarks demonstrate that GRiLS achieves improved mixing compared to existing gradient-based and gradient-free MCMC approaches.
\end{abstract}


\section{Introduction}

The accurate and efficient sampling of high-dimensional probability distributions is a central task in computational statistics. For sampling distributions on $\R^d$ whose density $\mu$ is known up to a normalization constant, Markov Chain Monte Carlo (MCMC) remains a predominant method~\cite{gelman1995bayesian, robert1999monte}.
Given a proposal density $q(\cdot|\cdot)$, MCMC algorithms build a Markov chain $\{x_1,x_2,\hdots\}$ by drawing a proposal candidate $x^\dagger \sim q(\cdot|  x_k )$ and directly accepting this as the next state $x_{k+1} = x^\dagger$ (unadjusted algorithms) or rejecting it with a certain probability and setting $x_{k+1} = x_k$ (Metropolis adjusted algorithms).
MCMC offers strong theoretical guarantees, especially when the target density $\mu$ is log-concave~\cite{durmus2017nonasymptotic, chewi2023log, dwivedi2019log}. However, one major difficulty in practical applications is when the target is not log-concave, e.g., under multimodality. In such scenarios, the Markov chain may be trapped in a mode, leading to poor mixing, high autocorrelation, and large bias. The computational efficiency of various samplers, often quantified by the time required to traverse energy barriers between modes, deteriorates exponentially as the modes become more separated~\cite{holley1989asymptotics}.

To overcome these mixing challenges, practitioners have explored a variety of techniques, including tempered MCMC (also known as Parallel Tempering~\cite{neal1996sampling}), transport map accelerated MCMC~\cite{parno2018transport,cui2023scalable}, adaptive biasing force methods~\cite{comer2015adaptive,stoltz2010free}, dimension reduction techniques~\cite{zahm2022certified,li2025sharp} and preconditioned Langevin dynamics~\cite{girolami2011riemann,xifara2014langevin,cui2024optimal,lelievre2025optimizing} to name just a few.
We focus here on the latter approach, which defines the proposal density %
via a time discretization of a Langevin dynamic that has been preconditioned in order to improve the convergence of the continuous-time dynamic toward equilibrium. This preconditioning is achieved by equipping $\R^d$ with a suitable Riemannian metric that locally reduces the geodesic distances between modes, thereby facilitating the transitions of particles across different regions of high probability.
Given an arbitrary field of symmetric definite matrices $W\colon x\mapsto W(x)\in\R^{d\times d}$, we endow $\R^d$ with the Riemannian metric $\langle u,v\rangle_x = u^\top W(x)^{-1} v$. The Riemannian Langevin dynamic is given by 
\begin{equation}\label{eq:LangevinPrecond_intro}
 \d X_t = \big( \mathrm{div}W(X_t) + W(X_t)\nabla \ln\mu(X_t) \big)\d t + \sqrt{2W(X_t)} \d B_t ,
\end{equation}
where $\mathrm{div}W(x)=(\sum_{j=1}^d \partial_j W_{i,j}(x))_{1\leq i \leq d}$ is the divergence of $W$ and $(B_t)_{t\geq0}$ is the standard Brownian motion in $\R^d$, see \emph{e.g.}~\cite{kent1978time,chung2013lectures}.
Popular choices for $W$
include the constant metric~\cite{hird2026non,haario2001adaptive,titsias2023optimal}, the inverse negative Hessian of the log-density~\cite{girolami2011riemann}, or a suitable positive-definite approximation of it when the target is not strongly log-concave~\cite{betancourt2013general,cao2025derivative}. More generally, the Hessian of any arbitrary strongly convex function may be used, giving rise to the so-called mirrored Langevin dynamics \cite{hsieh2018mirrored,zhang2020wasserstein}. It was proven in \cite{cui2024optimal} that an \emph{optimal} Riemannian metric exists (optimal in the sense of the fastest convergence of $X_t$ towards equilibrium) and can be expressed as the Hessian of a strongly convex function; however, its exact computation remains intractable in general.

In this work, we propose a gradient-free Riemannian Langevin sampler for targeting multimodal densities. To define the metric, we choose
\begin{equation}\label{eq:Wmu_Intro}
 W(x)=\frac{\nu(x)}{\mu(x)}\Sigma  ,
\end{equation}
where $\nu=\mathcal{N}(m,\Sigma)$ is the Gaussian density with same mean $m\in\R^d$ and covariance $\Sigma\in\R^{d\times d}$ as $\mu$.
Although it might not be optimal in the sense of \cite{cui2024optimal}, we show that this choice can greatly improve the convergence to equilibrium as compared to other choices, provided that $\nu$ ``covers'' the modes of $\mu$ by satisfying $\sup(\frac{\mu}{\nu})<\infty$. This choice has two key advantages. First, if a particle $x$ lies between two modes of $\mu$, we have $W(x)^{-1}\approx 0$ which reduces the geodesic distances between the modes. Second, the choice \eqref{eq:Wmu_Intro} leads to a \emph{gradient-free Langevin dynamic} which does not involve the term $\nabla\ln\mu(X_t)$ in~\eqref{eq:LangevinPrecond_intro}. This offers a significant advantage when gradients of the target densities are computationally prohibitive or intractable to evaluate.
Such gradient-free dynamics are receiving growing attention, see \cite{cui2024optimal} when $W$ is the optimal metric,~\cite{engquist2024adaptive, engquist2024sampling} for densities defined on the torus, and~\cite{kutri2026fast,garbuno2020affine,chakraborty2025affine} in the setting of ensemble-based Langevin samplers.

To define the proposal density based on the Riemannian Langevin dynamics \eqref{eq:LangevinPrecond_intro} with \eqref{eq:Wmu_Intro}, we will consider a Lamperti transform~\cite{lamperti1972semi} and an appropriate time discretization with stepsize $\Delta t > 0$. The resulting proposal sample $x^\dagger\sim q(\cdot | \theta_k)$ is given by:
\begin{align}\label{eq:Proposal_intro}
 x^\dagger &= \theta_k x_k + (1-\theta_k) m  + \sqrt{1-\theta_k^2} \, \xi_k ,
 \text{ where }
  \left\{
 \begin{array}{rl}
  \xi_k&\sim\mathcal{N}(0,\Sigma), \\
  \theta_k &= \exp\left(-\Delta t \frac{\nu(x_k)}{\mu(x_k)} \right), \\
  \nu(x) &\propto \exp\left(-\frac{1}{2}\|x - m\|_{\Sigma^{-1}}^2\right)  .
 \end{array}
 \right.
\end{align}
We denote the MCMC algorithm based on this proposal by Gradient-free Riemannian Langevin Sampler (GRiLS). For the proposal in~\eqref{eq:Proposal_intro}, the target density $\mu$ can be known up to a normalizing constant, since the latter can be absorbed in the stepsize $\Delta t$. Interestingly, letting $\theta_k = \theta_0 \in(0,1)$ be a constant parameter recovers the preconditioned Crank-Nicolson proposal (pCN)~\cite{cotter2013mcmc}. In this sense, the proposal in \eqref{eq:Proposal_intro} can be interpreted as an extension of the pCN scheme, which is consistent with the Langevin dynamics in~\eqref{eq:LangevinPrecond_intro} as $\Delta t\rightarrow0$.

In practice, we approximate the mean $m$ and covariance $\Sigma$ in~\eqref{eq:Proposal_intro} using approximate samples from $\mu$. These samples can either be obtained from previous MCMC iterations $\{x_1,\hdots,x_k\}$, following adaptive MCMC strategies~\cite{andrieu2008tutorial, haario2001adaptive,hird2026non}, or from an ensemble of $N\geq1$ particles $\{x_k^{1},\hdots,x_k^{N}\}$ at step $k$, as in more recent approaches~\cite{sprungk2025metropolis, garbuno2020interacting,carrillo2022consensus,leimkuhler2018ensemble}.
We focus on the ensemble-based approach and introduce a block-ensemble version of GRiLS (BE-GRiLS) in order to  improve the computational efficiency. Given a partition $\cup_{\ell=1}^P \mathcal{B}^{\ell} = \{1,\hdots,N\}$, each particle $x_k^{i}$ within a block $i\in\mathcal{B}^\ell$ is updated using the sample mean and covariance computed from all other blocks, that is:
\begin{equation}\label{eq:EnsembleMeanCov_intro}
 m^\ell_k = \frac{1}{N-\#\mathcal{B}^\ell}  \sum_{i\notin \mathcal{B}^\ell}  x_k^{i}
 \quad\text{and}\quad
 \Sigma^\ell_k =\frac{1}{N-\#\mathcal{B}^\ell}   \left( \sum_{i\notin \mathcal{B}^\ell}   (x_k^{i}) (x_k^{i})^\top \right) - (m^\ell_k)(m^\ell_k)^\top .
\end{equation}
As pointed out in \cite{sprungk2025metropolis}, this strategy enables the parallel update of all particles within a block, thereby improving the computational efficiency of the algorithm.
It is worth mentioning that the resulting sampler BE-GRiLS is similar to the Consensus-based sampler (CBS)~\cite{carrillo2022consensus} in several aspects, as they both use the ensemble mean and covariance to update each particle as in~\eqref{eq:Proposal_intro}.
The main difference is that, in CBS, the parameter $\theta_k=\theta_0$ is taken to be constant during the iterations and hence, as for pCN, CBS is not consistent with a Langevin dynamic targeting $\mu$ when $\Delta t \rightarrow 0$. Moreover, instead of using a Metropolis correction, CBS employs importance sampling to estimate the mean and covariance of $\mu$.

The rest of the paper is structured as follow.
In Section~\ref{sec:RiemannianLangevin} we motivate the choice of metric \eqref{eq:Wmu_Intro} for preconditioning the Riemannian Langevin dynamic \eqref{eq:LangevinPrecond_intro}.
In Section~\ref{sec:TimeIntegration_MCMCproposal} we propose a time-integration scheme which leads to the proposal in~\eqref{eq:Proposal_intro}.
In Section \ref{sec:SpectralAnalysis}, a spectral analysis of MCMC algorithms relates the convergence rate of MCMC algorithms with the one of the continuous Langevin dynamic~\eqref{eq:LangevinPrecond_intro}. Finally, in Sections~\ref{sec:NumericalResults} and~\ref{sec:HigherDimensional} we illustrate the performance of the resulting MCMC algorithm on several one-dimensional and multivariate benchmark problems.

\section{Riemannian Langevin Dynamics}\label{sec:RiemannianLangevin}

Let $\mu$ be a probability density on $\R^d$ defined by
$$
 \mu(x)\propto\exp(-V(x)),
$$
where $V:\R^d\rightarrow\R$ is a smooth potential function. A classic approach for sampling $\mu$ is to consider the overdamped Langevin dynamics
\begin{equation}\label{eq:Langevin}
 \d X_t = -\nabla V(X_t)\d t + \sqrt{2} \d B_t,
\end{equation}
where $B_t$ is the $d$-dimensional Brownian motion.
Under standard regularity assumptions on $V$, the invariant density of this dynamic is $\mu$, i.e., $X_t$ converges in law to $\mu$ as $t \rightarrow \infty$.
An alternative dynamic to \eqref{eq:Langevin} that does not change the invariant density is %
\begin{equation}\label{eq:LangevinPrecond}
 \d X_t = \big( \mathrm{div}W(X_t) - W(X_t)\nabla V(X_t) \big)\d t + \sqrt{2W(X_t)} \d B_t ,
\end{equation}
where $W:\R^d\rightarrow\mathcal{S}_{+}^d$ is any smooth field taking value in $\mathcal{S}_+^d\subset\R^{d\times d}$, the set of symmetric semi-definite positive matrices, see~\cite{kent1978time,chung2013lectures}.
Here, $\mathrm{div}W = (\sum_{j=1}^d \partial_j W_{i,j})_{1\leq i\leq d}$ is the divergence of $W$ and $\sqrt{W}$ is any square root of $W$ such that $\sqrt{W}(x)\sqrt{W}(x)^\top  = W(x)$. This dynamic corresponds to a Langevin diffusion process on the Riemannian manifold $\R^d$ endowed with the metric induced by $W^{-1}$, meaning $\langle u,v\rangle_X := u^\top W(X)^{-1} v$ for any $v,u\in \R^d$.
Over the past two decades, several choices have been proposed for $W$, ranging from constant matrix fields to location-dependent fields; see Table~\ref{tab:W} for a non-exhaustive overview of possible metrics.

\begin{table} \footnotesize
\hspace{-1cm}
\begin{tabular}{|p{0.01\textwidth}|>{\arraybackslash}p{0.33\textwidth}|>{\arraybackslash}p{0.65\textwidth}|}
\cline{2-3}
\multicolumn{1}{@{}p{0.05\textwidth}|}{} & \textbf{(inverse) Metric} & \textbf{Comments} \\ \hline
\multicolumn{1}{|@{}p{0.0\textwidth}|}{\multirow{3}{*}{\rotatebox{90}{constant}}}
& $W(x) = I_d$ & Standard Langevin dynamic \eqref{eq:Langevin} \\
& $W(x) = \Cov_\mu$ & Covariance-based preconditionning \cite{haario2001adaptive}, see \cite{garbuno2020affine} for an ensemble-based covariance estimation. \\
& $W(x) = \E_\mu [ \nabla\ln \mu \, \nabla\ln\mu^\top ]^{-1}$ & Inverse of Fisher matrix of $\mu$ \cite{titsias2023optimal} \\[3pt]
\hline
\multicolumn{1}{|@{}p{0.0\textwidth}|}{\multirow{3}{*}{\rotatebox{90}{$x$-dependent\hspace{1.2cm}}}}
&
$W(x) = ( \mathcal{I}(x) + \Sigma_0^{-1} )^{-1}$
& When $\d\mu(x)\propto\mathcal{L}^{y}(x)\d\mu_0(x)$ where $\mu_0=\mathcal{N}(0,\Sigma_0)$, use the Fisher information matrix $\mathcal{I}(x)=\E_{Y|x} [\nabla \ln\mathcal{L}^Y(x)\nabla \ln\mathcal{L}^Y(x)^\top ]$ of the likelihood $\mathcal{L}^y(x)$  \cite{girolami2011riemann}, see also \cite{kleppe2024log} for generalization to a class of latent variable models, and \cite{beskos2017geometric,cao2025derivative} for efficient approximate computation in high-dimension.
\\
&$W(x) = \mathrm{SoftAbs}[- \mathrm{Hess\,}\ln\mu(x)] $& Soft-absolute value of local Hessian of the log-density \cite{betancourt2013general} \\
& $W(x)=(\mathrm{Hess}\,\psi(x))^{-1}$
& Mirrored Langevin dynamic, for ad-hoc strictly convex $\psi$ \cite{hsieh2018mirrored,zhang2020wasserstein} \\
& $W(x) = \mu(x)^{-1}I_d$ & Derivative-free dynamics (for $\mu$ defined on the torus $\mathbb{T}^d$) \cite{engquist2024sampling,lelievre2025optimizing} \\
& $W \in \arg\min_{W:x\mapsto W(x)} C(\mu,W)$ & Optimizing the Poincar\'{e} constant  \cite{cui2024optimal,lelievre2025optimizing} \\[2pt] \cline{2-3}
& $W(x) = \frac{\nu(x)}{\mu(x)} \Cov_\mu$ & For $\nu= \mathcal{N}(m,\Sigma)$ with $m=\E_\mu[X]$ and $\Sigma = \Cov_\mu$ \textbf{(present paper)} \\
\hline
\end{tabular}
\caption{Possible choices of metric $W$}
\label{tab:W}
\end{table}

\subsection{Optimal Metric}

The idea proposed in~\cite{cui2024optimal, lelievre2025optimizing} is to identify $W$ %
by optimizing the convergence rate of the preconditioned dynamics in~\eqref{eq:LangevinPrecond}. Denoting the density of $X_t$ that solves~\eqref{eq:LangevinPrecond} by $\mu_t$, the chi-square divergence $\chi^2( \mu_t || \mu ):=\Var_\mu( \mu_t/\mu )$ satisfies
\begin{equation} \label{eq:Chisq_convergence}
 \chi^2( \mu_t || \mu ) \leq e^{-2t/C(\mu,W)} \chi^2( \mu_0 || \mu ),
\end{equation}
for any $t\geq0$ and any initial condition $\mu_0 \gg \mu$. Here, $C(\mu,W)\geq0$ denotes the Poincar\'{e} constant, which is defined as the smallest constant such that the Riemannian Poincar\'{e} inequality
\begin{equation} \label{eq:PI}
 \Var_\mu(f) \leq C(\mu,W) \,
 \int \| \nabla f(x)\|_{W(x)}^2 \mu(x)\d x ,
\end{equation}
holds for any smooth function $f:\R^d\rightarrow\R$; see~\cite{bakry2013analysis} for more details. Here we use the notation $\|v \|_A^2=v^\top A v$.
It turns out that $C(\mu,W)$ is also the smallest constant so that the inequality~\eqref{eq:Chisq_convergence} holds for any $t\geq0$ and any $\mu_0\gg\mu$, therefore it characterizes exactly the exponential convergence of the dynamic \eqref{eq:LangevinPrecond}.
Based on this observation, \cite{cui2024optimal,lelievre2025optimizing} propose to identify the matrix field $W$ which yields the best convergence rate by solving
\begin{equation}\label{eq:minC}
 \min_{\substack{ W:\R^d\rightarrow \mathcal{S}_+^d \\ \E_{\mu}[\trace(W)] = \trace(\Cov_\mu) }} C(\mu,W) .
\end{equation}
The constraint $\E_{\mu}[\trace(W)] = \trace(\Cov_\mu)$ is introduced in \cite{cui2024optimal} to fix the scaling of $W$ in order to prevent a trivial solution arising from $C(\mu,\alpha W) = \frac{1}{\alpha} C(\mu,W) \rightarrow 0$ as $\alpha\rightarrow\infty$.
Here, $\Cov_\mu=\int (x-m)(x-m)^\top \mu(x)\d x$ and $m=\int x\mu(x)\d x$ are the covariance and the mean of $\mu$.
While \cite{lelievre2025optimizing} considers alternative normalization constraints, the choice $\E_{\mu}[\trace(W)] = \trace(\Cov_\mu)$ offers significant advantages. First, by testing the Poincar\'{e} inequality \eqref{eq:PI} with affine functions $f$, this constraint yields the lower bound
\begin{equation}\label{eq:Cgeq1}
 C(\mu,W)\geq1 .
\end{equation}
Second, under some assumptions on $\mu$, it is shown in \cite[Section 2]{cui2024optimal} that an optimal solution $W_\mu^\text{opt}$ to \eqref{eq:minC} exists with $C(\mu,W_\mu^\text{opt})=1$. This optimal field $W_\mu^\text{opt}$ is given by
\begin{equation}\label{eq:Woptimal}
 W_\mu^\text{opt}(x) = \mathrm{Hess}\,\varphi^\star(x)^{-1},
\end{equation}
where $\varphi^\star:\text{supp}(\mu)\rightarrow\R$ is the strictly convex and smooth function corresponding to the convex conjugate of the \emph{moment map} $\varphi$ of $\mu$. We refer to~\cite{cordero2015moment} for more details on moment maps.
Because $W_\mu^\text{opt}$ is the inverse Hessian of a strictly convex function, the resulting dynamic \eqref{eq:LangevinPrecond} is a Mirrored Langevin dynamic \cite{hsieh2018mirrored,zhang2020wasserstein} of the form $X_t= \nabla \psi^\star(Y_t)$ with $\d Y_t = -\nabla V(X_t)\d t +\sqrt{2 \mathrm{Hess}\,\psi(x)} \d B_t$, with an \emph{optimal choice} for the mirror map $\psi=\varphi^\star$.
Third, \cite[Theorem 3.2.]{cui2024optimal} shows that $W_\mu^\text{opt}$ is actually a \emph{Stein Kernel}\footnote{A Stein Kernel for $\mu$ is any matrix field $W:\R^d\rightarrow\mathcal{S}_+^d$ such that $\int (x-m)f\d\mu = \int W\nabla f\d\mu $ for all smooth function $f:\R^d\rightarrow\R$, where $m=\int x\d\mu$.}, and therefore the dynamic \eqref{eq:LangevinPrecond} can be shown to further simplify as
$$
 \d X_t = -(X_t-m) \d t + \sqrt{2W_\mu^\text{opt}(X_t)} \d B_t ,
$$
where we recall that $m$ %
is the mean of $\mu$. Remarkably, this dynamic is gradient-free in the sense that it no longer involves the gradient of the potential $V$ of $\mu\propto\exp(-V)$.

Computing the optimal metric $W_\mu^\text{opt}$ (or the associated moment map $\varphi$) for a general measure $\mu$, however, is a difficult task. While 
\cite{cui2024optimal} proposes a gradient-descent algorithm for solving \eqref{eq:minC}, computing the gradient of $W\mapsto C(\mu,W)$ requires the eigendecomposition of a diffusion operator on $\R^d$, which is not computationally tractable for dimensions $d\geq3$. Nonetheless, there are two cases where $W_\mu^\text{opt}$ can be computed in closed form:
\begin{itemize}
 \item For $d=1$, if $\mu$ is supported on a convex domain, Proposition 2 in \cite{cui2024optimal} shows
 \begin{equation}\label{eq:Wstar_d1}
  W_\mu^\text{opt}(x) = \frac{1}{\mu(x)} \int_x^\infty (t-m) \mu(t)\d t,
  \qquad m=\int x\mu(x)\d x.
 \end{equation}
 See Figure \ref{fig:Wopt_VS_WGRiLS} for the representation of $W_\mu^\text{opt}$ when $\mu$ is a mixture of two Gaussian.
 This object has received growing attention in the probability literature~\cite{saumard2019weighted,germain2023note,ernst2020first} and in sensitivity analysis~\cite{heredia2025one,song2019derivative,roustant2025gradient,heredia2026weighted}.

 \item For a probability density $\mu$ defined on the torus $\mathbb{T}^d = (\R/\mathbb{Z})^d$, i.e., $\mu(x+k)=\mu(x)$ for all $x\in\R^d$ and $k\in \mathbb{Z}^d$, the optimal metric cannot be expressed as in \eqref{eq:Woptimal}. In fact, there is no moment map $\varphi$ for such $\mu$ because convex functions on $\mathbb{T}^d$ are necessarily constant. %
 Nonetheless,~\cite[Proposition 10]{lelievre2025optimizing} states that
  \begin{equation}\label{eq:Wstar_hom}
    W^{\text{opt,hom}}_{\mu}(x) = \frac{1}{\mu(x)} I_d,
  \end{equation}
  solves \eqref{eq:minC} in the \emph{homogenized limit}. That is, when replacing $\mu$ and $W$ in \eqref{eq:minC} with $\mu_{\sharp,k}(x)=\mu(kx)$ and $W_{\sharp,k}(x)=W(kx)$, respectively, and letting $k\in\mathbb{N}$ tend to $+\infty$ yields this solution; see \cite{allaire2002shape} for more details. The metric $W^{\text{opt,hom}}_{\mu}$ is used in~\cite{engquist2024sampling} to derive a gradient-free Langevin dynamic $\d X_t = \sqrt{2/\mu(X_t)}\d B_t$ from~\eqref{eq:LangevinPrecond}.

\end{itemize}

The two analytical solution~\eqref{eq:Wstar_d1} and~\eqref{eq:Wstar_hom} will guide our construction of a suboptimal, but computationally practical, metric in the next subsection.

\subsection{Gradient-Free Dynamics}

\begin{figure}
    \centering
    \includegraphics[width=0.6\textwidth]{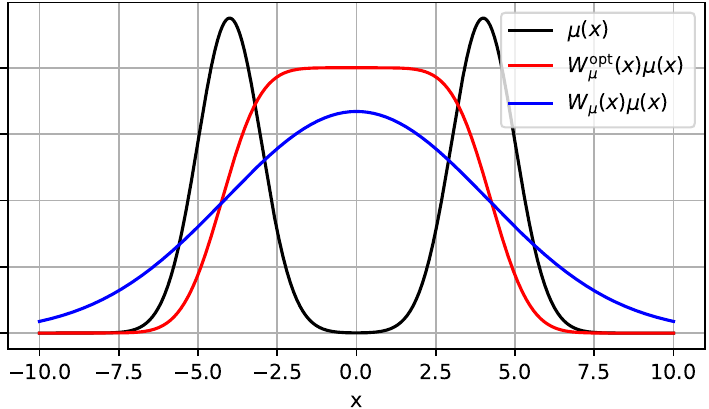}
    \caption{Optimal metric $W^\text{opt}_\mu$ as in \eqref{eq:Wstar_d1} (red curve) for the Gaussian mixture $\mu=\tfrac{1}{2}\mathcal{N}(-4,1)+\tfrac{1}{2}\mathcal{N}(+4,1)$ (bashed black curve).
    The blue curve represents the sub-optimal metric $W_\mu$ as in \eqref{eq:Wmu}.
    For better visualisation, the metrics are multiplied by $\mu$.
    }
    \label{fig:Wopt_VS_WGRiLS}
\end{figure}

For simplicity, we assume the probability density $\mu$ is fully supported on $\R^d$, so that $\mu(x)>0$ for all $x\in\R^d$.
With $\mu(x)\propto e^{-V(x)}$, the drift term in \eqref{eq:LangevinPrecond} can then be written as
\begin{equation}\label{eq:LangevinPrecond_drift}
  \mathrm{div}W(x) - W(x)\nabla V(x)
  = \frac{1}{\mu(x)} \mathrm{div} ( W(x) \mu(x) ).
\end{equation}
Thus, any field of the form $W(x)=A(x)/\mu(x)$ for some field $A:\R^d\rightarrow\mathcal{S}_+^d$ results in a Langevin dynamic~\eqref{eq:LangevinPrecond} in which the gradient of $V$ does not appear: only the divergence of $A$ is required.
This is the case for the optimal metrics \eqref{eq:Wstar_d1} and \eqref{eq:Wstar_hom} mentioned earlier.

We now construct a simple and tractable field $A$ that yields a suitable preconditioner for the Langevin dynamic.
Based on Figure \ref{fig:Wopt_VS_WGRiLS}, two important features of $W_\mu^\text{opt}$ are that $W_\mu^\text{opt}(x)\mu(x)$ is unimodal and that it encompass the support of $\mu$.
Hence, we propose
\begin{equation}\label{eq:Wmu}
  W_\mu(x) =
  \frac{\nu(x)}{\mu(x)}
  \Sigma,
  \quad\text{where}\quad
  \left\{\begin{array}{rl}
          \nu(x) &= \frac{\exp\left(-\frac{1}{2}\|x-m\|^2_{\Sigma^{-1}}\right)}{\sqrt{2\pi\det{\Sigma}}},\\
          m &= \E_{\mu}[X] ,\\
          \Sigma &= \Cov_\mu .
        \end{array}
 \right.
\end{equation}
Figure~\ref{fig:Wopt_VS_WGRiLS} presents $W_\mu$ and the optimal $W_\mu^\text{opt}$ given by \eqref{eq:Wmu} and \eqref{eq:Wstar_d1}, respectively.
While the choice in~\eqref{eq:Wmu} %
 might not be optimal with respect to optimization problem~\eqref{eq:minC}, it is a natural and convenient choice in practice.
In addition, setting $\Sigma=\Cov_\mu$ ensures that $W_\mu$ satisfies the normalization constraint of \eqref{eq:minC} by construction, meaning $\E_\mu[\trace(W_\mu)]=\E_\nu[\trace( \Sigma )]=\trace(\Cov_\mu)$.
Combining \eqref{eq:LangevinPrecond_drift} with \eqref{eq:Wmu}, the preconditioned dynamic in~\eqref{eq:LangevinPrecond} simplifies to
\begin{equation}\label{eq:LangevinGaussian}
 \d X_t = - \frac{\nu(X_t)}{\mu(X_t)} (X_t-m) \d t   +  \sqrt{2\frac{\nu(X_t)}{\mu(X_t)} \Sigma} \d B_t .
\end{equation}
This Langevin dynamic does not involve the gradient of the potential $V$, %
while still admitting $\mu$ as its invariant measure by construction.
The following proposition provides a simple upper bound for the Poincar\'{e} constant $C(\mu,W_\mu)$.

\begin{proposition}\label{prop:boundC}
 Let $\mu$ be a probability density that is fully supported on $\R^d$, and let $W_\mu$ take the form in~\eqref{eq:Wmu}. Then
 \begin{equation}\label{eq:boundC}
   1\leq  C\left(\mu,  W_\mu \right) \leq  \sup_{x\in\R^d}\frac{\mu(x)}{\nu(x)}.
 \end{equation}
\end{proposition}
\begin{proof}
The left inequality $1\leq  C\left(\mu,  W_\mu \right)$ is a direct consequence of the normalization $\E_\mu[\trace(W_\mu)]=\trace(\Cov_\mu)$; see \eqref{eq:Cgeq1}.
For the upper bound, we use that
$$\Var_\mu(f) = \min_{\alpha\in\R} \int (f-\alpha)^2  \d\mu \leq (\sup\tfrac{\mu}{\nu}) \min_{\alpha\in\R}\int (f-\alpha)^2  \d\nu = (\sup\tfrac{\mu}{\nu}) \Var_\nu(f),$$
for any $f$.
Then, the Gaussian Poincar\'{e} inequality $\Var_\nu(f)\leq \E_\nu[\|\nabla f\|^2_{\Sigma}]$ (see \cite{bakry2013analysis}) yields
$
 \Var_\mu(f)
  \leq(\sup\tfrac{\mu}{\nu}) \E_\nu[\|\nabla f \|^2_{\Sigma}] = (\sup\tfrac{\mu}{\nu}) \E_\mu[\|\nabla f \|^2_{W_\mu}]
$, which proves \eqref{eq:boundC}.
\end{proof}

We now illustrate Proposition \ref{prop:boundC} on the one-dimensional Gaussian mixture
$$\mu_h = \frac{1}{2} \mathcal{N}(-h,1) +\frac{1}{2} \mathcal{N}(+h,1) ,$$
for some $h\geq0$. It is well known that, for such multimodal densities, the Poincar\'{e} constant associated with the standard Langevin dynamic \eqref{eq:Langevin} grows exponentially with the height of the energy barrier between modes (here, $h^2$); see~\cite{menz2014poincare,holley1989asymptotics}.
We show in Appendix \ref{proof:Cexplosion} that the lower bound
\begin{equation}\label{eq:Cexplosion}
  C(\mu_h, \Cov_{\mu_h} )
  \geq \frac{ e^{h^2/3} }{1+h^2}  ,
\end{equation}
holds for any $h\geq 1/2$.
The sharpest upper bound for $C(\mu_h, \Cov_{\mu_h} )$ to the best of our knowledge is derived in~\cite[Section 4.1]{schlichting2019poincare} and yields $ C(\mu_h, \Cov_{\mu_h}) \leq \frac{3+e^{4h^2}}{4(1+h^2)}$.
Instead, by considering the metric $W_{\mu_h}$ in~\eqref{eq:Wmu}, the Poincar\'{e} constant $C(\mu_h, W_{\mu_h} )$ associated with the preconditioned Langevin dynamic \eqref{eq:LangevinGaussian} satisfies
 \begin{equation}\label{eq:GaussianMixture_1d}
  C(\mu_h, W_{\mu_h} )
  \overset{\eqref{eq:boundC}}{\leq} \sup_{x\in\R^d}\frac{\mu_h(x)}{\nu_h(x)}
  \leq \exp(1) \sqrt{1+h^2},
 \end{equation}
for any $h\geq0$, where $\nu_h=\mathcal{N}(\E_\mu[X],\Cov_{\mu_h})=\mathcal{N}(0,1+h^2)$. %
We give a proof of inequality~\eqref{eq:GaussianMixture_1d} in Appendix \ref{proof:GaussianMixture}.
Both $W_{\mu_h}$ and $\Cov_{\mu_h}$ satisfy the normalization constraint $\int\trace(W_{\mu_h})\d\mu_h = \trace(\Cov_{\mu_h})$. Therefore, they can be compared as suboptimal solutions to \eqref{eq:minC}: while the Poincar\'{e} constant associated with the standard Langevin dynamic~\eqref{eq:Langevin} grows \emph{at least exponentially} in $h^2$, the one of the preconditioned Langevin dynamic~\eqref{eq:LangevinGaussian} growths \emph{at most linearly} in $h$ for $h \gg 1$.

\begin{remark}[Generalization to high-dimensional mixtures] 
 The bound \eqref{eq:GaussianMixture_1d} generalizes to Gaussian mixtures with $N\geq2$ components in dimension $d\geq1$. Consider $\mu_h^N(x) = \frac{1}{N} \sum_{i=1}^N \mu_i(x)$ where $\mu_i = \mathcal{N}(m_i,I_d)$. Letting $h\geq0$ be the smallest constant such that $\|m_i-m\|\leq h$ for all $i$, where $m=\frac{1}{N}\sum_{i=1}^N m_i$ is the mean of $\mu_h^N$, we have
 \begin{equation}\label{eq:GaussianMixture_dd}
  C(\mu,W_{\mu_h^N}) \leq \exp(N/2)(1+h^2)^{d/2}.
 \end{equation}
 The proof is given in Appendix \ref{proof:GaussianMixture}.
 While this bound still grows polynomially in $h$, it depends exponentially on $d$ and $N$. We believe this behavior is an artifact of the proof technique, which relies on the bound \eqref{eq:boundC}. This bound is likely overly pessimistic, particularly in high-dimensional settings.

\end{remark}

\begin{remark}[Generalization to non-Gaussian reference densities]\label{rmk:nonGaussianReference}
The previous development can be generalized to a non-Gaussian reference density $\nu$.
Given an approximation $\nu$ to $\mu$, the choice
$$
  W_\mu^\nu(x) =
  \Omega
  \frac{\nu(x)}{\mu(x)}
  \Cov_\nu ,
  \quad\text{where }
  \Omega= \frac{\trace(\Cov_\mu)}{\trace(\Cov_\nu)} ,
$$
yields the following preconditioned Langevin dynamic
$$
 \d X_t = - \Omega\frac{\nu(X_t)}{\mu(X_t)} \Cov_\nu\nabla\log\nu(X_t) \d t   +  \sqrt{2\Omega\frac{\nu(X_t)}{\mu(X_t)} \Cov_\nu} \d B_t .
$$
This dynamic does not involve the gradient of the target $\mu$, but only the one of the approximation $\nu$.
By construction we have $\E_\mu[\trace(W_\mu^\nu)]=\trace(\Cov_\mu)$ so that $W_\mu^\nu$ satisfies the normalization constraint in \eqref{eq:minC}.
Thus, similarly to Proposition~\ref{prop:boundC}, we can show that
$1\leq C(\mu,W_\mu^\nu)\leq \tfrac{\trace(\Cov_\mu)}{\trace(\Cov_\nu)}(\sup \tfrac{\mu}{\nu}) C(\nu,\Cov_\nu)$.
This analysis suggests to finding $\nu$ by minimizing $\nu\rightarrow \tfrac{\trace(\Cov_\mu)}{\trace(\Cov_\nu)}(\sup \tfrac{\mu}{\nu})$. We leave this for future work.

\end{remark}

\section{Sampling Algorithms}\label{sec:TimeIntegration_MCMCproposal}

In this section we consider different numerical time integration methods for the dynamics~\eqref{eq:LangevinGaussian} in order to construct a MCMC proposal. Given a stepsize $\Delta t>0$, the proposal approximately draws a sample $x_{k+1}$ from $X_{\Delta t}$, where $X_t$ is the solution to \eqref{eq:LangevinGaussian} initialized at $X_0=x_k$.
A first approach is to apply the Euler-Maruyama scheme directly to \eqref{eq:LangevinGaussian}, leading to
\begin{equation}\label{eq:EulerMaruyama}
 x_{k+1} = x_{k} - \frac{\nu(x_k)}{\mu(x_k)}(x_k-m) \Delta t + \sqrt{2\Delta t\frac{\nu(x_k)}{\mu(x_k)} } \xi_k ,  \qquad \xi_k\sim\mathcal{N}(0,\Sigma) .
\end{equation}
This naive discretization is problematic when the ratio $\nu/\mu$ exhibits strong spatial variations. Indeed, if $x_k$ is between two modes of $\mu$, we have $\mu(x_k)\ll \Delta t\nu(x_k)$. Hence, the mean and covariance of the Gaussian vector $x_{k+1}|x_k$ defined above will explode. The following Lamperti transformation avoids this behavior.

\begin{proposition}[Lamperti transformation]\label{prop:Lamperti}
Let $\mu$ be a probability density that is fully supported on $\R^d$, and let $\nu=\mathcal{N}(m,\Sigma)$ be the Gaussian density with mean $m$ and covariance $\Sigma$.
Let $Y_t$ be the solution of the Ornstein-Uhlenbeck process
\begin{equation}\label{eq:OU}
 \d Y_t = -(Y_t-m) \d t   +  \sqrt{2 \Sigma} \d B_t ,
\end{equation}
with initial condition $X_0=Y_0$, and consider the random variable
\begin{equation}\label{eq:StochasticTimeChange}
  \tau(t) = \int_{0}^t  \frac{\nu(Y_{s})}{\mu(Y_{s})}  \d s .
\end{equation}
Then, for all $t\geq0$, we have $Y_{\tau(t)} \overset{d}{=} X_t$, where $(X_t)_{t\geq0}$ is the solution to \eqref{eq:LangevinGaussian}.
\end{proposition}
\begin{proof}
 Applying the stochastic time change in~\cite[Theorem 8.5.1]{oksendal2003stochastic} yields the result; see also~\cite{oksendal1990stochastic} for the proof.
\end{proof}

Proposition \ref{prop:Lamperti} provides an alternative way to draw a sample $x_{k+1}$ from $X_{\Delta t} |X_0=x_k$ by following the steps: (i) draw a trajectory $(Y_t)_{t\geq0}$ of the Ornstein-Uhlenbeck (OU) process \eqref{eq:OU} initialized at $Y_0=x_k$ and (ii) compute the (deterministic) integral $\tau(\Delta t)$ in \eqref{eq:StochasticTimeChange} conditioned on $(Y_t)_{t\geq0}$. Since Proposition \ref{prop:Lamperti} ensures $Y_{\tau(\Delta t)} $ has the same law as $X_{\Delta t} |X_0=x_k$, we can set $x_{k+1}=Y_{\tau(\Delta t)}$.
In practice, we need to numerically approximate $\tau(\Delta t)$ in step (ii). We propose to use the first-order quadrature scheme:
\begin{equation}\label{eq:tau_quadrature}
 \tau(\Delta t)
 = \int_{0}^{\Delta t}  \frac{\nu(Y_{s})}{\mu(Y_{s})}  \d s
 \quad\approx\quad
 \Delta t \frac{\nu(Y_0)}{\mu(Y_0)}
 = \Delta t \frac{\nu(x_k)}{\mu(x_k)} .
\end{equation}
Notably, the quadrature scheme does not depend on the values of $Y_t$ for $t>0$.
Next, using the closed form solution of the OU process
$
 Y_{s}|Y_0
 \sim \mathcal{N} \left( m+e^{-s}(Y_0-m)  ,  (1-e^{-2s})\Sigma \right)
$, we can draw a sample $x_{k+1}$ with law $(Y_{\Delta t \nu(x_k)/\mu(x_k) }| Y_0=x_k)$ as follows
\begin{align}\label{eq:LampertiScheme}
 x_{k+1} &= m + \theta_k (x_k-m) + \sqrt{1-\theta_k^2} \xi_k, \qquad \text{where }
 \begin{cases}
\xi_k\sim\mathcal{N}(0,\Sigma) ,\\
  \theta_k = \exp(- \Delta t \frac{\nu(x_k)}{\mu(x_k)} ) .
 \end{cases}
\end{align}
Contrarily to the Euler-Maruyama scheme \eqref{eq:EulerMaruyama}, $\mu(x_k)\ll \Delta t\nu(x_k)$ implies $\theta_k\ll1$ so that the Gaussian vector $x_{k+1}|x_k$ defined in \eqref{eq:LampertiScheme} will be approximately drawn from $\mathcal{N}(m,\Sigma)$, which has finite mean and covariance.
Notably, as $\Delta t\rightarrow0$, the Taylor expansion $\theta_k = 1-\Delta t\frac{\nu(x_k)}{\mu(x_k)} + \mathcal{O}(\Delta t^2)$ shows that \eqref{eq:LampertiScheme} recovers the Euler-Maruyama scheme \eqref{eq:EulerMaruyama} up to second order terms $\mathcal{O}(\Delta t^2)$.

\begin{remark}\label{rmk:HigherOrderLamperti}
 Classical error analysis of quadrature scheme~\eqref{eq:tau_quadrature} yields
 $$\E\left[\left(\tau(\Delta t)-\Delta t\frac{\nu(Y_0)}{\mu(Y_0)}\right)^2\right]^{1/2}=\mathcal{O}(\Delta t^{3/2}).$$
 Higher order quadrature methods can be used to improve the estimation of $\tau(\Delta t)$ at the cost of evaluating the gradient of the integrand $\nu/\mu$.
 We refer to \cite[Chapter 5]{kloeden2013numerical} for more details.
\end{remark}

\begin{remark}
 The Lamperti transform can be related to an importance sampling strategy. Indeed, for any integrable function $f$, we can write
\begin{align*}
 \int f(x)\mu(x)\d x
 &= \int \left(f(y)\frac{\mu(y)}{\nu(y)}\right) \nu(y)\d y \\
 (\text{For } Y_t \text{ solving } \eqref{eq:OU})\quad &= \lim_{T\rightarrow\infty} \frac{1}{T} \int_{t=0}^T f(Y_t)\frac{\mu(Y_t)}{\nu(Y_t)} \d t \\
 (t\leftarrow \tau(t) \text{ as in } \eqref{eq:StochasticTimeChange}) \quad&= \lim_{T\rightarrow\infty} \frac{1}{T} \int_{t=0}^{\tau^{-1}(T)} f(Y_{\tau(t)}) \d t  \\
  &= \lim_{T\rightarrow\infty} \frac{1}{T} \int_{t=0}^T f(Y_{\tau(t)}) \d t
\end{align*}
where, for the last step, we used the fact that $\tau^{-1}(T) \sim T$ when $T\rightarrow \infty$
(i.e., by letting $f(x)=1$, %
we deduce $1=\int f(x)\mu(x)\d x = \lim_{T\rightarrow \infty} \tau^{-1}(T)/T$ and thus $\tau^{-1}(T) \sim T$).
Instead of using the ratio $\mu(y)/\nu(y)$ as an importance weight, as classically done by the importance sampling method, the Lamperti transform stretches time via $t\leftarrow \tau(t)$ in order to correct for the bias introduced by using $Y_t$ in place of $X_t \overset{d}{=} Y_{\tau(t)}$.
\end{remark}

\subsection{MCMC Proposal}

Given a target probability density $\mu$ (possibly unnormalized) and a proposal density $q(\cdot|\cdot)$, our MCMC algorithm builds a Markov chain $\{x_0,x_1,\hdots\}\subset\R^d$ by iteratively proposing a candidate $x^\dagger \sim q(\cdot|x_k)$. For an unadjusted algorithm, one always accepts this candidate $x_{k+1}=x^\dagger$. For Metropolis adjusted algorithms, one accepts the candidate $x_{k+1}=x^\dagger$ with probability
\begin{equation}\label{eq:alpha}
 \alpha(x^\dagger|x_k) = \min\left\{1 , \frac{\mu(x^\dagger) q(x_k|x^\dagger)}{\mu(x_k)q(x^\dagger|x_k)}\right\}.
\end{equation}
If the candidate is rejected, then $x_{k+1}=x_k$. This accept/reject step ensures that the Markov chain $\{x_k\}_{k\geq0}$ admits $\mu$ as an invariant measure \cite{gelman1995bayesian, robert1999monte}.
Motivated by the scheme in~\eqref{eq:LampertiScheme}, we introduce the proposal density
\begin{equation}\label{eq:GRiLS_proposal}
 q(x^\dagger|x_k)
 = \frac{1}{\sqrt{2\pi \mathrm{det}((1-\theta_k^2)\Sigma) } } \exp\left(-\frac{\| x^\dagger - m - \theta_k (x_k-m) \|^2_{\Sigma^{-1}}}{2(1-\theta_k^2)}\right),
\end{equation}
where $\theta_k = \exp( -\Delta t \frac{\nu(x_k) }{\mu(x_k)} )$.
In this definition, $m$ and $\Sigma$ are not necessarily the mean and the covariance of the density $\mu$, which are typically unknown.
In practice, $m$ and $\Sigma$ can be defined as the sample mean and sample covariance as described later in Section \ref{sec:BE-GRiLS}. Alternatively, as done in \cite{cotter2013mcmc}, when $\mu\propto\mathcal{L}\mu_0$ is the posterior density of a Bayesian inverse problem with prior $\mu_0$ and likelihood $\mathcal{L}$, then $m$ and $\Sigma$ can be chosen to be the mean and covariance of $\mu_0$.
With this proposal density, the acceptance rate simplifies to
\begin{align}\label{eq:GRiLS_AcceptanceRate}
 \alpha(x^\dagger|x_k)
 &= \min\left\{1 ~,~
 \frac{\mu(x^\dagger)}{\mu(x_k)}
 \left(\frac{1-\theta^{\dagger2}}{1-\theta_k^2}\right)^{d/2}
 \frac{ \exp\left(-\frac{\| x^\dagger - m - \theta_k (x_k-m) \|^2_{\Sigma^{-1}}}{2(1-\theta_k^2)}\right)}{ \exp\left(-\frac{\| x_k - m - \theta^\dagger (x^\dagger-m) \|^2_{\Sigma^{-1}}}{2(1-\theta^{\dagger2})}\right)}\right\} ,
\end{align}
where
$
 \theta^\dagger = \exp( -\Delta t \frac{\nu(x^\dagger) }{\mu(x^\dagger)}  ).
$
The resulting MCMC algorithm is described in Algorithm \ref{alg:GRiLS}.

Let us make a few comments on this algorithm.
First, in order to compute $\theta_k$ and $\theta^\dagger$, there is no need to know the normalizing constant of $\mu$ nor of $\nu$ as they can both be absorbed in the parameter $\Delta t$.
Second, as mentioned earlier, the proposal \eqref{eq:GRiLS_proposal} is similar to the preconditioned Crank-Nicolson proposal (pCN) introduced in \cite{cotter2013mcmc}. Indeed, replacing $\theta_k$ with a constant value $\theta\in (0,1)$ yields the pCN proposal for which the acceptance probability simplifies to $\alpha(x^\dagger|x_k) = \min \{1 ; \frac{\mu(x^\dagger)\nu(x_k)}{\mu(x_k)\nu(x^\dagger)} \} $.
Third, the acceptance rate exhibits an exponential dependence on the dimension $d$. This behavior is the counterpart of using a location-dependent proposal variance. As the following proposition demonstrates, the effect of this dimensional dependence can be mitigated when $\Delta t$ is chosen sufficiently large.

\begin{proposition}\label{prop:controltheta}
 Let $\mu$ and $\nu$ be two probability densities on $\R^d$ (possibly unnormalized) such that $\Omega = \sup\frac{\mu}{\nu}<\infty$.
 Then, for any $\beta>1$, taking
 $\Delta t \geq \frac{\Omega}{2} \ln( 1+\frac{ d  }{2\ln(\beta)})$
 ensures
 $$
  \beta^{-1} \leq \left(\frac{1-\theta(x)^2}{1-\theta(y)^2}\right)^{d/2}\leq \beta ,
 $$
 for any $x,y\in\R^d$, where $\theta(x)=\exp(-\Delta t \nu(x)/\mu(x))$.
\end{proposition}

\begin{proof}
By definition of $\Omega = \sup\frac{\mu}{\nu}$ we have $\theta(x)\leq \exp(-\Delta t/\Omega)$ for all $x\in\R^d$. Then
$$
 \left(\frac{1-\exp(-\frac{2\Delta t}{\Omega})}{1}\right)^{d/2}
 \leq
 \left(\frac{1-\theta(x)^2}{1-\theta(y)^2}\right)^{d/2}
 \leq
 \left(\frac{1}{1-\exp(-\frac{2\Delta t}{\Omega})}\right)^{d/2}.
$$
It remains to show that $(1-\exp(-\frac{2\Delta t}{\Omega}))^{-d/2} \leq \beta$, which is equivalent to $ \Delta t \geq -\frac{\Omega}{2}\ln(1-\beta^{-2/d}) $.
Using the inequality $1-\exp(-u)\geq u/(1+u)$, we have
\begin{align*}
  -\ln(1-\beta^{-2/d})
  = -\ln\left(1-\exp\left(-\tfrac{2\ln(\beta)}{d}\right)\right)
  \leq -\ln\left( \frac{\frac{2\ln(\beta)}{d}}{ 1+\frac{2\ln(\beta)}{d} }\right)
  = \ln\left( 1+\frac{ d  }{2\ln(\beta)}\right) .
\end{align*}
Then, the assumption $ \Delta t \geq \frac{\Omega}{2} \ln( 1+\frac{ d  }{2\ln(\beta)})$ is sufficient to ensure $\Delta t \geq -\frac{\Omega}{2}\ln(1-\beta^{-2/d})$, which concludes the proof.
\end{proof}

\begin{remark}[Independence Sampler]\label{rmk:IS}
 While Proposition~\ref{prop:controltheta} shows that for large enough $\Delta t$, the ratio $(\frac{1-\theta^{\dagger2}}{1-\theta_k^2})^{d/2}$ in \eqref{eq:GRiLS_AcceptanceRate} can be uniformly bounded, it is also worth mentioning that $\Delta t\gg\Omega=\sup\frac{\mu}{\nu}$ simplifies drastically the proposal density. Indeed, we have $\theta_k\leq \exp(-\Delta t/\Omega) \rightarrow 0$ uniformly on $k$ when $\Delta t\rightarrow \infty$, and therefore
 \begin{equation}\label{eq:IndependenceSampler}
  q(\cdot|x_k)
  \underset{\Delta t\rightarrow \infty}{\overset{d}{\longrightarrow}}
  \nu(\cdot)=\mathcal{N}(m,\Sigma).
 \end{equation}
 In this case, GRiLS becomes an Independence Sampler (IS) in which the proposal candidate is drawn from $\nu$ independently of the current state. This type of sampler is efficient if the proposal density $\nu$ is close to $\mu$. Using the formula $\min\{a,b\}=(a+b-|a-b|)/2$, the mean acceptance rate is given by
 \begin{align}
  \alpha
  &\;\;= \E[ \alpha(Y|X) ] ,\qquad
  \begin{cases}
   X\sim\mu(\cdot) \\ Y\sim q(\cdot|X)
  \end{cases} \nonumber\\
  &\;\overset{\eqref{eq:alpha}}{=} \int \min\{ \mu(x)q(y|x) ;  \mu(y)q(x|y) \} \d x \d y \nonumber\\
  &\overset{\eqref{eq:IndependenceSampler}}{\underset{\Delta t\rightarrow \infty}{\longrightarrow}} \int \min\{ \mu(x)\nu(y) ;  \mu(y)\nu(x) \} \d x \d y \nonumber\\
  &\;\;= 1- \frac{1}{2}\int \Big| \mu(x) \nu(y)  - \mu(y) \nu(x) \Big| \d x \d y \nonumber\\
  &\;\;= 1- \| \mu\otimes\nu  - \nu\otimes\mu \|_{\text{TV}} , \label{eq:Alpha_IS}
 \end{align}
 where $\|\cdot\|_{\text{TV}}$ denotes the total variation distance. We observe numerically in Section \ref{sec:NumericalResults} that the mean acceptance rate of GRiLS is always greater than the above limit.
\end{remark}

\begin{algorithm}[!ht]
\caption{GRiLS (Gradient-free Riemannian Langevin Sampler)}\label{alg:GRiLS}
\KwIn{Initial state $x_0\in\R^d$,
target probability density $\mu$ (possibly unnormalized),
mean $m\in\R^d$,
covariance $\Sigma\in\mathcal{S}^d_{+}$,
step size $\Delta t>0$,
number of steps $K\in\mathbb{N}$,
}
\KwOut{MCMC chain $\{x_1,\hdots,x_K\} \subset\R^d$
\vspace{0.3cm}}

\textbf{Algorithm} \nolinkurl{GRiLS}$( x_0 , \mu, m , \Sigma , \Delta t , K)$

\For{$k = 0$ \KwTo $K-1$}{

    Compute
    $
     \theta_k = \exp( -\Delta t \frac{\nu(x_k) }{\mu(x_k)}  ) ,
    $
    where $\nu(x) = \exp(-\|x-m\|^2_{\Sigma^{-1}})$

    Draw a sample $\xi_k\sim\mathcal{N}(0,\Sigma)$ and build the candidate proposal
    $$
     x^\dagger = m + \theta_k (x_k-m) + \sqrt{1-\theta_k^2} \xi_k ,
    $$

    Compute
    $
     \theta^\dagger = \exp ( -\Delta t \frac{\nu(x^\dagger) }{\mu(x^\dagger)}  )
    $
    and, with probability
    $$
    \alpha(x^\dagger|x_k) = \min\left\{1 ; \frac{\mu(x^\dagger)}{\mu(x_k)}\left(\frac{1-\theta^{\dagger2}}{1-\theta_k^2}\right)^{d/2}\frac{ \exp\left(-\frac{\| x^\dagger - m - \theta_k (x_k-m) \|^2_{\Sigma^{-1}}}{2(1-\theta_k^2)}\right)}{ \exp\left(-\frac{\| x_k - m - \theta^\dagger (x^\dagger-m) \|^2_{\Sigma^{-1}}}{2(1-\theta^{\dagger2})}\right)}\right\}
    $$

    \textbf{accept} the candidate by setting $x_{k+1}=x^\dagger$
    or \textbf{reject} by setting $x_{k+1}=x_k$.
}
\Return{$\{x_1,\hdots,x_K\}$}\;
\end{algorithm}

\subsection{Block Ensemble GRiLS}\label{sec:BE-GRiLS}

In this section, we propose an ensemble version of GRiLS in which we estimate the mean $m$ and the covariance $\Sigma$ of $\mu$ using an ensemble $\{x^1,\hdots,x^N\}$ of $N$ particles in $\R^d$. For simplicity, we represent this ensemble by the matrix
$$
 X = (x^1,\hdots,x^N) \in \R^{d\times N}.
$$
We aim to build a MCMC chain $\{X_1,X_2,\hdots\}\subset\R^{d\times N}$ that targets the probability density in which the members are independent and drawn
from the tensor product measure of the target
\begin{equation}\label{eq:GibbsTarget}
 \mu^{\otimes N}(X) = \mu(x^1)\cdots \mu(x^N).
\end{equation}
To do this, we consider a Metropolis-within-Gibbs algorithm, where $X_{k+1}$ is constructed by updating each column of $X_k$ one after the other.
One Gibbs iteration starts by initializing $X_{k\rightarrow k+1} = X_k$ and, for $\ell=1 \hdots N$, updates the $\ell$-th column of $X_{k\rightarrow k+1}$ to
\begin{equation}\label{eq:Gibbs}
 X_{k\rightarrow k+1}
 \leftarrow
 (x^1_{k+1},\hdots , x^{\ell-1}_{k+1} , x^{\ell}_{k+1} , x^{\ell+1}_{k} ,\hdots, x^{N}_{k} ) .
\end{equation}
At this step, $x^{\ell}_{k+1}$ can be generated by any Metropolis algorithm targeting the conditional density $\mu^{\otimes N}(x^\ell|x^1_{k+1},\hdots , x^{\ell-1}_{k+1} , x^{\ell+1}_{k} ,\hdots, x^{N}_{k})$ which, by \eqref{eq:GibbsTarget}, simplifies to $\mu(x^\ell)$.
After the $N$ columns are updated, we set $X_{k+1} = X_{k\rightarrow k+1}$ and we move on to the next Gibbs iteration $k\leftarrow k+1$.

We now detail the Metropolis step which generates the update $x^{\ell}_{k+1}$ in \eqref{eq:Gibbs}. For simplicity, we denote the current ensemble $X_{k\rightarrow k+1}$ excluding its $\ell$-th element by
\begin{equation}\label{eq:Ensemble_Minus_ell}
 X_{k\rightarrow k+1}^{-\ell}=(x^1_{k+1},\hdots , x^{\ell-1}_{k+1} , x^{\ell+1}_{k} ,\hdots, x^{N}_{k} ) \in\R^{d\times(N-1)}.
\end{equation}
Given a proposal density of the form $q(x^{\dagger}|  x^{\ell}_{k} , X_{k\rightarrow k+1}^{-\ell} )$, we draw a candidate sample $x^{\dagger}\sim q(\cdot|  x^{\ell}_{k} , X_{k\rightarrow k+1}^{-\ell} )$ and accept it in the ensemble $x_{k+1}^\ell=x^{\dagger}$ with probability
\begin{align}
 \alpha(x^{\dagger}|x^{\ell}_{k}, X_{k\rightarrow k+1}^{-\ell})
 &=
 \min\left\{1;
 \frac{\mu( x^{\dagger} )}{\mu( x^{\ell}_k )}
 \frac{q( x^{\ell}_k| x^{\dagger} , X_{k\rightarrow k+1}^{-\ell} )}{q(x^{\dagger}|  x^{\ell}_k , X_{k\rightarrow k+1}^{-\ell} )}
 \right\} .
 \label{eq:acceptance_MetroWithinGibbs}
\end{align}
Otherwise, we reject $x^{\dagger}$ and set $x_{k+1}^\ell=x_{k}^\ell$.
Defined that way, this Metropolis-within-Gibbs algorithm guarantees detailed balance, which ensures $\mu^{\otimes N}$ is the stationary density of the chain $\{X_k\}_{k\geq1}$; see~\cite{roberts2006harris} for more details.

We propose to employ the GRiLS proposal \eqref{eq:GRiLS_proposal} in which $m$ and $\Sigma$ are defined as the sample mean and the sample covariance of the ensemble $X_{k\rightarrow k+1}^{-\ell}$ defined in \eqref{eq:Ensemble_Minus_ell}.
The resulting proposal is given by
\begin{equation}\label{eq:GRiLS_Consensus_proposal}
 q( x^\dagger  | x^\ell_k , X^{-\ell}_{k\rightarrow k+1} )
 =
 \frac{1}{\sqrt{2\pi \mathrm{det}((1-(\theta_k^\ell)^2) \Sigma^\ell_k ) } } \exp\left(-\frac{\| x^\dagger - m^\ell_k - \theta_k^\ell (x^\ell_k- m^\ell_k ) \|^2_{(\Sigma^\ell_k)^{-1}}}{2(1-(\theta_k^\ell)^2)}\right),
\end{equation}
where $\theta_k^\ell = \exp(-\Delta t \nu(x_k^\ell) / \mu(x^\ell_k) )$. Using $\mathbf{1}_{N-1}\in\R^{N-1}$ to denote the vector of ones, the mean and covariance of the proposal can be computed as
\begin{align}
 m^{\ell}_k = \frac{(X^{\ell-1}_{k\rightarrow k+1}) \mathbf{1}_{N-1}}{N-1}
 \quad\text{and}\quad
 \Sigma_{k}^{\ell} &= \frac{(X^{\ell-1}_{k\rightarrow k+1})(X^{\ell-1}_{k\rightarrow k+1})^\top }{N-1}
 -  (m^{\ell}_k)(m^{\ell}_k)^\top.
 \label{eq:SigmaEnsemble}
\end{align}
By excluding the $\ell$-th member in the computation of $m^{\ell}_k$ and $\Sigma_{k}^{\ell}$, the acceptance probability $\alpha(x^{\dagger}|x^{\ell}_{k}, X_{k\rightarrow k+1}^{-\ell})$ is the same as the one in~\eqref{eq:GRiLS_AcceptanceRate} with $m=m^{\ell}_k$ and $\Sigma=\Sigma_{k}^{\ell}$.

So far, we have updated one particle at a time $x_{k}^\ell \rightarrow x_{k+1}^\ell$.
In the same way, we now derive an algorithm which updates multiple particles \emph{in parallel} $\{x_{k}^i\}_{\ell\in\mathcal{B}} \rightarrow \{x_{k+1}^i\}_{\ell\in\mathcal{B}}$ for some subset $\mathcal{B}\subset\{1,\hdots,N\}$, thereby reducing the number of iterations by a factor $\#\mathcal{B}$ and improving the overall computational efficiency. To preserve detailed balance for the resulting MCMC algorithm, we estimate the mean and the covariance using the particles that are not in $\mathcal{B}$. We refer to this algorithm as Block-Ensemble GRiLS (BE-GRiLS).
The steps of the procedure are summarized in Algorithm~\ref{alg:BE-GRiLS}.

\begin{remark}[Avoiding covariance matrix factorizations]
The ensemble covariance matrix $\Sigma_{k}^{\ell}$ in~\eqref{eq:SigmaEnsemble} can be expressed equivalently in matrix form as
 $$
 \Sigma_{k}^{\ell}
 = \frac{(X^{\ell-1}_{k\rightarrow k+1})   ( I_{N-1} - \frac{\mathbf{1}_{N-1}\mathbf{1}_{N-1}^\top }{N-1})  (X^{\ell-1}_{k\rightarrow k+1})  ^\top}{N-1} ,
 $$
 where $ I_{N-1}$ is the identity matrix of size $N-1$.
 This expression allows for straightforward sampling of a Gaussian vector $\xi_k^\ell\sim\mathcal{N}(0,\Sigma^{\ell}_k)$, which is required to generate the candidate proposal $x^\dagger = m^{\ell}_k + \theta_k (x_k-m^{\ell}_k) + \sqrt{1-\theta_k^2} \xi_k^\ell$.
 In particular, for $Z\sim\mathcal{N}(0,I_{N-1})$, we have
 $$
  \xi_k^\ell
  = (X^{\ell-1}_{k\rightarrow k+1}) \frac{  ( I_{N-1} - \frac{\mathbf{1}_{N-1}\mathbf{1}_{N-1}^\top }{N-1})  }{\sqrt{N-1}} Z
  \quad\sim\quad \mathcal{N}(0,\Sigma_{k}^{\ell}).
 $$
 Thus, we can sample from $\mathcal{N}(0,\Sigma_{k}^{\ell})$ without computing any factorization of $\Sigma_{k}^{\ell}$.
\end{remark}

\begin{algorithm}[!ht]
\caption{BE-GRiLS (Block-Ensemble GRiLS)}\label{alg:BE-GRiLS}
\SetKwFor{ForComment}{For}{}{end}

\KwIn{Initial ensemble $X_0=(x^1_0,\hdots,x^N_0)\in\R^{d\times N}$,
target probability density $\mu$ (possibly unnormalized),
partition $0=N_0 \leq \hdots\leq N_{P}=N $,
step size $\Delta t>0$,
number of steps $K\in\mathbb{N}$,
}
\KwOut{MCMC chain $\{X_1,\hdots,X_K\}\subset\R^{d\times N}$
\vspace{0.3cm}}

\textbf{Algorithm} \nolinkurl{BE-GRiLS}$( X_0 , \mu, (N_\ell)_{\ell=0}^P,\Delta t , K)$

\ForComment{$i \le d$ \textnormal{\textbf{do} \hfill \ttfamily{\#\# Gibbs iteration}}}{

    Initialize  $X_{k\rightarrow k+1} := X_k $

    \ForComment{$\ell=1$ \KwTo $P$ \textnormal{\textbf{do} \hfill \ttfamily{\#\# Update the $\ell$-th block}}}{

    Remove the $\ell$-th block of ensemble from $X_{k\rightarrow k+1}$

    $$
     X_{k\rightarrow k+1}^{-\ell}=(x^1_{k+1},\hdots , x^{N_{\ell-1}}_{k+1} , x^{(N_{\ell})+1}_{k} ,\hdots, x^{N}_{k} )
    $$

    Compute $N(\ell)=N-(N_{\ell}-N_{\ell-1})$ and the ensemble mean and covariance
    \begin{align*}
      m^{\ell}_k = \frac{(X^{\ell-1}_{k\rightarrow k+1}) \mathbf{1}_{N(\ell)}}{N(\ell)}
      \quad\text{and}\quad
      \Sigma_{k}^{\ell}
      = \frac{(X^{\ell-1}_{k\rightarrow k+1})(X^{\ell-1}_{k\rightarrow k+1})^\top }{N(\ell)}
      -  (m^{\ell}_k)(m^{\ell}_k)^\top
    \end{align*}

    \ForComment{$i = (N_{\ell-1})+1$ \KwTo $N_{\ell}$ \textnormal{\textbf{do} \hfill \ttfamily{\#\# Metropolis (parallelizable)}}}{
    $
    x_{k+1}^i = \text{\nolinkurl{GRiLS}}( x_{k}^i, \mu, m_{k}^{\ell} , \Sigma_{k}^{\ell} , \Delta t , 1)
    $

    }
    Update the $\ell$-th block of ensemble
    $$
     X_{k\rightarrow k+1}=(x^1_{k+1},\hdots  , x^{N_{\ell}}_{k+1}
     , x^{(N_{\ell})+1}_{k} ,\hdots, x^{N}_{k} )
    $$
    }

    Update the ensemble $X_{k+1} = X_{k\rightarrow k+1}$

}
\Return{$\{X_1,\hdots,X_K\}$}\;
\end{algorithm}

\newpage
\section{Spectral Analysis of MCMC Algorithms}\label{sec:SpectralAnalysis}

In this section we compare the efficiency of GRiLS against other MCMC algorithms via a spectral analysis of their respective transition kernels.
The transition kernel $P(\cdot,\cdot)$ of a homogeneous Markov chain $\{x_0,x_1,x_2,\hdots\}\subset\R^d$ is a function defined on $\R^d\times \mathcal{B}(\R^d)$ such that $P(x,\cdot)$ is the probability measure of $x_{k+1}|x_{k}=x$ for all $x\in\R^d$.
Let $\mu_{k}$ denote the probability measure of $x_k\sim\mu_k$. Then, we can write
\begin{equation}\label{eq:transition_MCMC}
 \mu_{k+1} = \mu_k P,
\end{equation}
where we use the notation $\rho P(\cdot) = \int_{x\in\R^d}  P(x,\cdot) \rho(\d x)$ for any probability measure $\rho$ on $\mathcal{B}(\R^d)$.
We assume that there exists a unique invariant measure $\mu_{\infty}$ such that $\mu_{\infty} = \mu_\infty P$.
The transition kernel of unadjusted MCMC algorithms with proposal density $q(\cdot|\cdot)$ is then
\begin{equation}\label{eq:Kernel_MCMC_unadjusted}
  P(x,\d y)=q(y|x)\d y,
\end{equation}
whose invariant measure is not necessarily equal to $\mu$ in general. 
For Metropolis adjusted MCMC algorithms with proposal density $q(\cdot|\cdot)$ targeting a probability density $\mu$, the transition kernel is
\begin{equation}\label{eq:Kernel_MCMC_adjusted}
 P(x,\d y) = q(y|x) \alpha(y|x)\d y + \delta_{x}(\d y) \int \Big(1-\alpha(y'|x) \Big) q(y'|x) \d y' ,
\end{equation}
where $\delta_x$ is the Dirac measure centered at $x$ and $\alpha(y|x)=\min\{1;\frac{\mu(y) q(x|y)}{\mu(x)q(y|x)}\}$.
This definition of $\alpha$ implies detailed balance $\mu(x)  P(x,\d y) \d x= \mu( y) P(y,\d x)\d y$, which ensures $\mu_\infty=\mu$; see for instance~\cite{tierney1998note}.
To analyze the convergence of $\mu_{k}$ towards $\mu_\infty$, we consider the linear operator $f\mapsto Pf$ on $L^2_{\mu_\infty}$ defined by
$$
 Pf(x) = \int_{y\in\R^d} f(y) P(x,\d y).
$$
This operator has many interesting properties that are detailed in~\cite[Chapter 12]{levin2017markov}.
First, together with the property $\mu_\infty=\mu_\infty P$, Jensen's inequality allows us to write $\int (Pf)^2 \d\mu_\infty \leq \int P(f^2) \d\mu_\infty =\int f^2 \d\mu_\infty $ so that $f\mapsto Pf$ is a contraction in $L^2_{\mu_\infty}$.
Second, detailed balance ensures the operator is symmetric in $L^2_{\mu_\infty}$, meaning $\int g (Pf)\d\mu_\infty = \int (Pg)f\d\mu_\infty$.
Third, since $P(\cdot|x)$ is a probability measure for all $x$, we have $P1=1$.
Finally, denoting by $f_k(x)=\frac{\mu_k(x)}{\mu_\infty(x)}$ the density of $\mu_k$ with respect to $\mu_\infty$, the detailed balance (again) permits us to write
$$
 Pf_{k}(x)
 = \int_{y\in\R^d} \mu_{k}(y)\frac{P(x,\d y)}{\mu_\infty(y)}
 = \int_{y\in\R^d} \mu_{k}(y)\frac{P(y,\d x)\d y}{\mu_\infty(x)\d x}
 = \frac{\mu_{k+1}(\d x)}{\mu_\infty(x)\d x}
 = f_{k+1}(x).
$$
To analyze the convergence of $\mu_{k}$ towards $\mu_\infty$, we consider the chi-square divergence $\chi^2(\mu_{k+1} || \mu_\infty) = \Var_\mu(\mu_{k+1}/\mu_\infty) $ of $\mu_{k+1}$ from $\mu_\infty$. Because $P(f_{k} -1)= f_{k+1} -1$, we can write
\begin{align*}
 \chi^2(\mu_{k+1} || \mu_\infty)
 &= \int \left( P( f_{k} -1) \right)^2 \d\mu_\infty \\
 &= \left( \frac{\int \left( P(f_{k} -1) \right)^2 \d\mu_\infty}{\int \left( f_{k} -1 \right)^2 \d\mu_\infty} \right) \chi^2(\mu_{k} || \mu_\infty)  \\
 &\leq \left( \sup_{\substack{ f \in L^2_{\mu_\infty} \text{ s.t.} \\ \int f\d \mu_\infty=0 }}
  \frac{\int \left( Pf \right)^2 \d\mu_\infty}{\int f^2 \d\mu_\infty}
 \right) \chi^2(\mu_{k} || \mu_\infty)  \\
 &= (1-\mathrm{Gap}(P))^2 \chi^2(\mu_{k} || \mu_\infty) ,
\end{align*}
hence the geometric convergence $\chi^2(\mu_{k} || \mu_\infty) \leq (1-\mathrm{Gap}(P))^{2k} \chi^2(\mu_{0} || \mu_\infty)$. Here, $\mathrm{Gap}(P)$ denotes the spectral gap of $P$ which is defined by $\mathrm{Gap}(P) = 1-\lambda_2$ where $\lambda_2$ is the 2nd largest eigenvalue of $P$ seen as an operator in $L^2_{\mu_\infty}$ (the largest eigenvalue begin $\lambda_1=1$).
An equivalent expression for the spectral gap is
$$
\mathrm{Gap}(P)=\inf_{ f \in L^2_{\mu_\infty} } \frac{\mathcal{E}(f,f)}{\Var_{\mu_\infty}(f)},
$$
where $\mathcal{E}(f,f) $ is the Dirichlet form associated with $P$ defined by
\begin{align}
 \mathcal{E}(f,f)
 &:= \int f^2- f(P f) \d\mu_\infty \nonumber\\
 &= \frac{1}{2}\int f^2-2 f(P  f)+ P (f^2) \d\mu_\infty  \nonumber\\
 &= \frac{1}{2}\int\int \Big(f(x)-f(y)\Big)^2 P (x,\d y) \mu_\infty (x)\d x . \label{eq:Dirichlet}
\end{align}
We show in Section \ref{sec:Ulam} how to numerically estimate $\mathrm{Gap}(P)$ in dimension $d=1$.
\begin{remark}[Spectral Gap and Poincar\'{e} constant]
 The above expression \eqref{eq:Dirichlet} permits to link the spectral gap with the Poincar\'{e} constant.
 Indeed, a Taylor expansion on $f$ permits to write
 $ ( f(x)-f(y) )^2  = ( \nabla f(x)^\top(y-x) )^2 + \mathcal{O}(\|x-y\|^{3})  $ so that
 \begin{align*}
 \mathcal{E}(f,f)
 &= \frac{1}{2}\int  \nabla f(x)^\top \Sigma(x) \nabla f(x) \mu_\infty (x)\d x +  \mathcal{O}(\delta) ,
\end{align*}
where $\Sigma(x) = \int (y-x)(y-x)^\top P(x,\d y)$ and $\delta =  \int\int \|x-y\|^{3}  P(x,\d y)\mu_\infty (x)\d x$.
 For instance, for the unadjusted algorithm \eqref{eq:Kernel_MCMC_unadjusted} with proposal density
 \begin{equation}\label{eq:GenericGaussianProposal}
  q_{\Delta t}(\cdot|x) = \mathcal{N}(m_{\Delta t}(x), \Sigma_{\Delta t}(x))
  \quad\text{where}\quad
  \begin{cases}
   m_{\Delta t}(x) = x + \Delta t b(x)+ \mathcal{O}(\Delta t^{3/2}) \\
   \Sigma_{\Delta t}(x) = 2\Delta t W(x) + \mathcal{O}(\Delta t^{3/2}) ,
  \end{cases}
 \end{equation}
 with $b:\R^d\rightarrow\R^d$ and $W:\R^d\rightarrow\mathcal{S}_{+}^d$ two arbitrary functions, we have
 $\Sigma(x) = 2\Delta t W(x)+\mathcal{O}(\Delta t^{3/2}) $ and $
   \delta = \mathcal{O}(\Delta t^{3/2})$.
 Thus, if $\mu_\infty\overset{d}{\rightarrow}\mu$ as $\Delta t\rightarrow0$ (which is the case when $b(x) = \mathrm{div}W(x) - W(x)\nabla V(x)$, see \emph{e.g.} \cite[Theorem 5.1]{mattingly2010convergence}), we deduce
 \begin{equation}\label{eq:Dirichlet_Poincare}
  \mathcal{E}(f,f) = \Delta t \E_\mu[ \| \nabla f \|_{W}^2 ] +  \mathcal{O}(\Delta t^{3/2}).
 \end{equation}
 A similar expression can be obtain with the Metropolis adjusted algorithm \eqref{eq:Kernel_MCMC_adjusted} with the same proposal density $q_{\Delta t}(\cdot|\cdot)$.
 Let us emphasis that the constants hidden in the term $\mathcal{O}(\Delta t^{3/2})$ depend on $f$: if we had the stronger statement
 $\mathcal{E}(f,f) = \Delta t \E_\mu[ \| \nabla f \|_{W}^2 ] +  \Delta t^{3/2} \mathcal{R}(f)$ with $\mathcal{R}(f)\leq C(1+ \E_\mu[ \| \nabla f \|_{W}^2 ] )$ for all $f\in L^2_\mu$, we would be able to conclude that
 \begin{equation}\label{eq:GAP_VS_Poincare}
  \mathrm{Gap}(P) = \frac{\Delta t}{ C(\mu,W)} + \mathcal{O}(\Delta t^{3/2}) ,
 \end{equation}
 where $C(\mu,W)$ is the Poincar\'{e} constant such that $\Var_\mu(f) \leq C(\mu,W)  \E_\mu[\| \nabla f\|_{W}^2]$ for all smooth function $f$.
 A precise derivation of \eqref{eq:GAP_VS_Poincare} is beyond the scope of this paper and is left for future work.

\end{remark}

\begin{remark}
The spectral gap relates to the integrated autocorrelation time (IACT) as
$$
 \mathrm{Gap}(P) = \frac{2}{1+\sup_{f\in L^2_{\mu_{\infty}}}  \mathrm{IACT}(f)} ,
$$
where
$$
 \mathrm{IACT}(f) = 1+2\sum_{k=1}^\infty \frac{\Cov( f(x_0),f(x_k) )}{\Var(f(x_0))} ,
$$
see for instance \cite{geyer1992practical}.
In principle, a sample-based estimator of the IACT allows one to approximate $\mathrm{IACT}(f) \approx \widehat{\mathrm{IACT}}(f)$ for any (fixed) function $f\in L^2_{\mu_{\infty}}$.
However, estimating $\sup_{f\in L^2_{\mu_{\infty}}}  \mathrm{IACT}(f)$ is more challenging, since one must estimate $\mathrm{IACT}(f)$ simultaneously for several functions $f$.
For instance, consider the finite-dimensional subspace $V_N =\text{span}\{ \mathbf{1}_{\Omega_i} ,\hdots, \mathbf{1}_{\Omega_N} \}$ of piecewise constant functions where $\Omega_1,\hdots,\Omega_N$ is a partition of $\R^d$. Accurate estimation of $\sup_{f\in V_N}  \mathrm{IACT}(f)$ requires a sufficient number of samples in each subdomain $\Omega_i$, which becomes impractical when $N$ is large. We show in Section \ref{sec:Ulam} an alternative approach for computing $\mathrm{Gap}(P)$ directly in dimension $d=1$.
\end{remark}

\section{Numerical Experiments in Dimension One}\label{sec:NumericalResults}

\begin{table}[]
\centering
\small
\begin{tabular}{|l|l|l|}
\hline
 \textbf{Acronym} & \textbf{Candidate} $x^\dagger\sim q(\cdot|x_k)$ (with $\xi_k\overset{\text{iid}}{\sim}\mathcal{N}(0,\Sigma)$) & \textbf{Parameters} \\ \hline
 IS & $x^\dagger =  m+\xi_k $  &   $\Sigma\succ0, m\in\R^d$ \\ \hline
 RW& $x^\dagger = x_k + \sqrt{2\Delta t} \xi_k $  &   $\Sigma\succ0, \Delta t >0$ \\ \hline
 MALA (ULA)        & $x^\dagger = x_k  + \Delta t\Sigma\nabla\ln\mu(x_k) + \sqrt{2\Delta t} \xi_k$ & $\Sigma\succ0,\Delta t  > 0$ \\ \hline
 pCN   & $x^\dagger = \theta x_k + (1-\theta) m  + \sqrt{1-\theta^2} \xi_k$,  & $\Sigma\succ0,m\in\R^d,\Delta t>0$ \\ & where $\theta=\exp(-\Delta t) $& \\ \hline
 GRiLS  & $x^\dagger = \theta_k x_k + (1-\theta_k) m  + \sqrt{1-\theta_k^2} \xi_k $, & $\Sigma\succ0,m\in\R^d,\Delta t > 0$ \\
 (U-GRiLS) &where $\theta_k = \exp(-\Delta t \tfrac{\nu(x_k)}{\mu(x_k)} )$ and $\nu=\mathcal{N}(m,\Sigma)$ & \\ \hline
\end{tabular}
\caption{\underline{Metropolis-adjusted MCMC algorithms}: Independence Sampler (\textbf{IS}), Random walk (\textbf{RW}),  Metropolis Adjusted Langevin Algorithm (\textbf{MALA}), preconditionned Crank-Nicolson (\textbf{pCN}) and Gradient-free Riemannian Langevin Sampler (\textbf{GRiLS}).
\underline{Unadjusted algorithms}: Unadjusted Langevin Algorithm (\textbf{ULA}) and unadjusted GRiLS (\textbf{U-GRiLS}).}
\label{tab:std_proposal}
\end{table}

We compare numerically the the MCMC algorithms listed in Table \ref{tab:std_proposal} on two target measures in dimension one\footnote{The implementation to reproduce the numerical results can be found at:~\url{https://gitlab.inria.fr/ozahm/grils_dim1}.}. We consider a Gaussian mixture $\mu_h$ and a smoothed piecewise constant density $\mu_\varepsilon$ defined respectively by
\begin{align*}
 \mu_h(x) &\propto \mathbf{1}_{[-10,10]}(x) \left( \exp\left( -\frac{(x-h/2)^2}{2}\right) + \exp\left( -\frac{(x+h/2)^2}{2}\right) \right)\\
 \mu_\varepsilon(x) &\propto \mathbf{1}_{[-2,2]}(x) \sum_{i=1}^4 \alpha_i \left( \sigma_\varepsilon(x-r_i) - \sigma_\varepsilon(x-r_{i+1}) \right) .
\end{align*}
Here $\sigma_\varepsilon(t) = \frac{1}{1+\exp(-t/\varepsilon)}$ is a sigmoid function (with the convention $\sigma_{\varepsilon=0}(t) = 1+\tfrac{1}{2}\text{sign}(t)$), and $\alpha=(0.5,4,1,3)$ and
$r=(-2,1,0,1,2)$. As illustrated on Figure \ref{fig:TargetDensities_d1}, the parameters $h$ and $\varepsilon$ control the difficulty of the problem: large values of $h\gg1$ enforce stronger multimodality in the Gaussian mixture and small values of $\varepsilon\ll 1$ make $\mu_\varepsilon$ closer to a piecewise-constant measure. In these regimes, standard Langevin-based algorithms are expected to perform poorly.

In these experiments, the mean $m$ and the covariance $\Sigma$ are computed analytically (see Section \ref{sec:HigherDimensional} for the use of Algorithm \ref{alg:BE-GRiLS} to estimate both $m$ and $\Sigma$ from interacting particles).

\begin{figure}
    \centering
    \begin{subfigure}[b]{0.48\textwidth}
        \includegraphics[width=\textwidth]{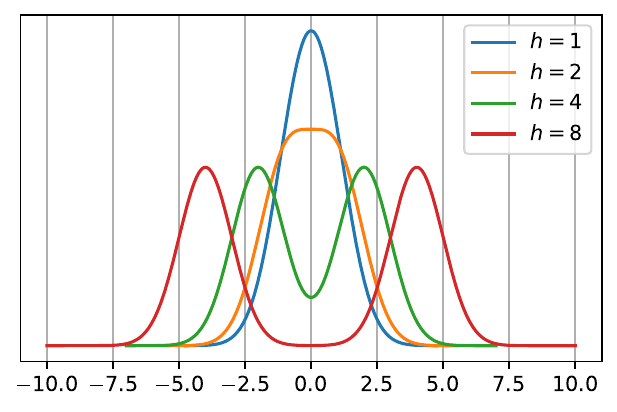}
        \caption{Mixture of two Gaussian $\mu_h$}
    \end{subfigure}
    \hfill
    \begin{subfigure}[b]{0.48\textwidth}
        \includegraphics[width=\textwidth]{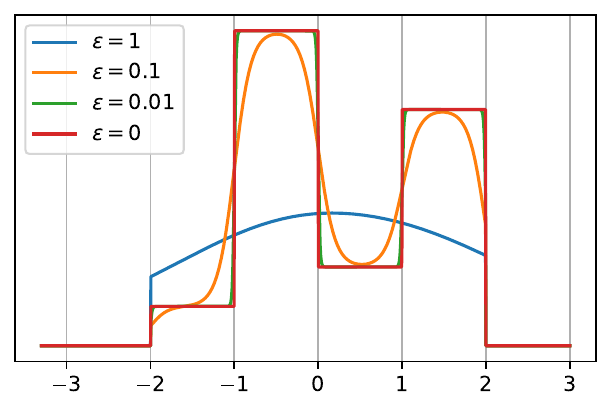}
        \caption{Smoothed piecewise constant density $\mu_\varepsilon$}
    \end{subfigure}
    \caption{Two target densities in dimension $d=1$. Left: Mixture of two Gaussian $\mu_h = \frac{1}{2}\mathcal{N}(-h/2,1) + \frac{1}{2}\mathcal{N}(+h/2,1)$}
    \label{fig:TargetDensities_d1}
\end{figure}

\subsection{Ulam's Method}\label{sec:Ulam}

In order to numerically estimate $\mu_\infty$ (for unadjusted algorithms), the expected acceptance rate (for adjusted algorithms) and $\mathrm{Gap}(P)$, we employ a piecewise constant discretization scheme.
Let $[a,b]=\text{supp}(\mu)$ be the support of the target measure and let $c_i = a + (b-a) \frac{i-1}{N-1}$ for $1\leq i \leq N$. In all our experiments, we take $N=2000$. We consider the partition $\R = \cup_{i=1}^N \Omega_i$ where $\Omega_{i}=[ \tfrac{c_{i-1}+c_i}{2} ,  \tfrac{c_{1}+c_{i+1}}{2} ]$ with the convention $c_0=-\infty$ and $c_{N+1}=+\infty$.
We then assemble the matrix
\begin{align*}
 Q_{ij}^N = \int_{\Omega_j} q(y|c_i)\d y ,
\end{align*}
by using the closed form expression of the CDF of the Gaussian density $q(\cdot|c_i)$.
By construction, $Q^N=(Q^N_{ij})$ is a stochastic matrix (all its entries are nonnegative and each row sums to one) which approximates the \textbf{unadjusted algorithm}.
By the Frobenius-Perron theorem, there is a nonnegative vector $v\in\R^N$ whose components sum to one, such that $v=(Q^N)^\top v$, and the corresponding density function
$$
 \mu_\infty^N(x) = \sum_{i=1}^N  \frac{v_i}{ |\Omega_i| }  \mathbf{1}_{\Omega_i}(x) ,
$$
is a piecewise constant approximation to $\mu_\infty$ called the Ulam approximation \cite{ulam1960collection}.
Although this method is simple, its convergence is usually slow since it only uses piecewise constant functions, typically $\|\mu_\infty - \mu_{\infty}^N\|_{L^1} = \mathcal{O}( N^{-1}\ln(N))$, see \cite{froyland1998approximating}.
In addition, the spectral gap $\mathrm{Gap}(Q^N) = 1-\lambda_2^N$ can be computed by finding the second largest eigenvalue $\lambda_2^N$ of $(Q^N)^\top w = \lambda_2^N w $.

The stochastic matrix representing the \textbf{Metropolis adjusted algorithm} targeting $\mu^N(x) = \sum_{i=1}^N \frac{\mu(c_i)}{|\Omega_i|}\mathbf{1}_{\Omega_i}(x)$ is assembled via
$$
 P^N_{ij} = \begin{cases}
             \min\{ \mu(c_i) Q_{ij}^N , \mu(c_j) Q_{ji}^N \} \mu(c_i)^{-1}  &  i\neq j ,\\
             1-\sum_{k\neq i} P_{ik}^N & i=j .
            \end{cases}
$$
The spectral gap $\mathrm{Gap}(P^N) = 1-\lambda_2^N$ is also be computed by finding the second largest eigenvalue $\lambda_2^N$ of $(P^N)^\top w = \lambda_2^N w $, and the expected acceptance rate is given by
$$
 \E[\alpha^N] = \sum_{i,j=1}^N \min\{ \mu(c_i) Q_{ij}^N , \mu(c_j) Q_{ji}^N \} .
$$

\subsection{Unadjusted Algorithms}

We compare here ULA and U-GRiLS. Because both algorithms are based on a time-discretization of a Langevin dynamic, their stationary measure $\mu_\infty^N$ converges to $\mu^N(x)= \sum_{i=1}^N \frac{\mu(c_i)}{|\Omega_i|}\mathbf{1}_{\Omega_i}(x)$ as $\chi^2( \mu^N || \mu_{\infty}^N) = \mathcal{O}(\sqrt{\Delta t})$, see \cite[Theorem 5.1]{mattingly2010convergence}. In addition, their spectral gap decreases with $\Delta t$ as $\mathrm{Gap}(Q^N)=\mathcal{O}(\Delta t)$ according to \eqref{eq:GAP_VS_Poincare}.
Thus we expect
\begin{equation}\label{eq:GAP_BIAS_unadjusted}
 \mathrm{Gap}(Q^N) = \mathcal{O}\Big( \chi^2( \mu^N || \mu_{\infty}^N)^2 \Big) .
\end{equation}
Figure \ref{fig:C_vs_bias} represents the spectral gap as a function of the bias, both quantities being computed for a range of $\Delta t\in [10^{-4},10^0]$.
Note that relation \eqref{eq:GAP_BIAS_unadjusted} is observed for U-GRiLS on $\mu_h$ for $h=8$ and on $\mu_\varepsilon$ with $\varepsilon\in\{0.01,0.001,0\}$.
In addition, in these plots, a larger spectral gap $\mathrm{Gap}(Q^N)$ indicates faster convergence toward the stationary measure $\mu_\infty^N$, assuming a comparable bias. Overall, we observe that GRiLS often outperforms ULA.

\begin{figure}
    \centering
    \includegraphics[width=\textwidth]{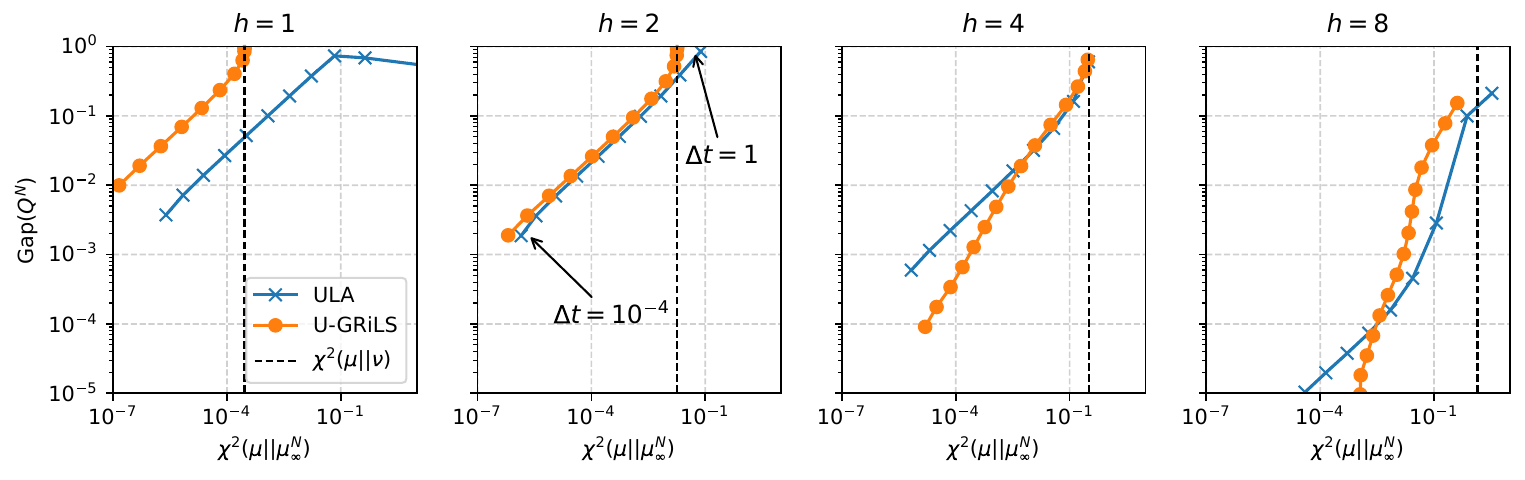}
    \includegraphics[width=\textwidth]{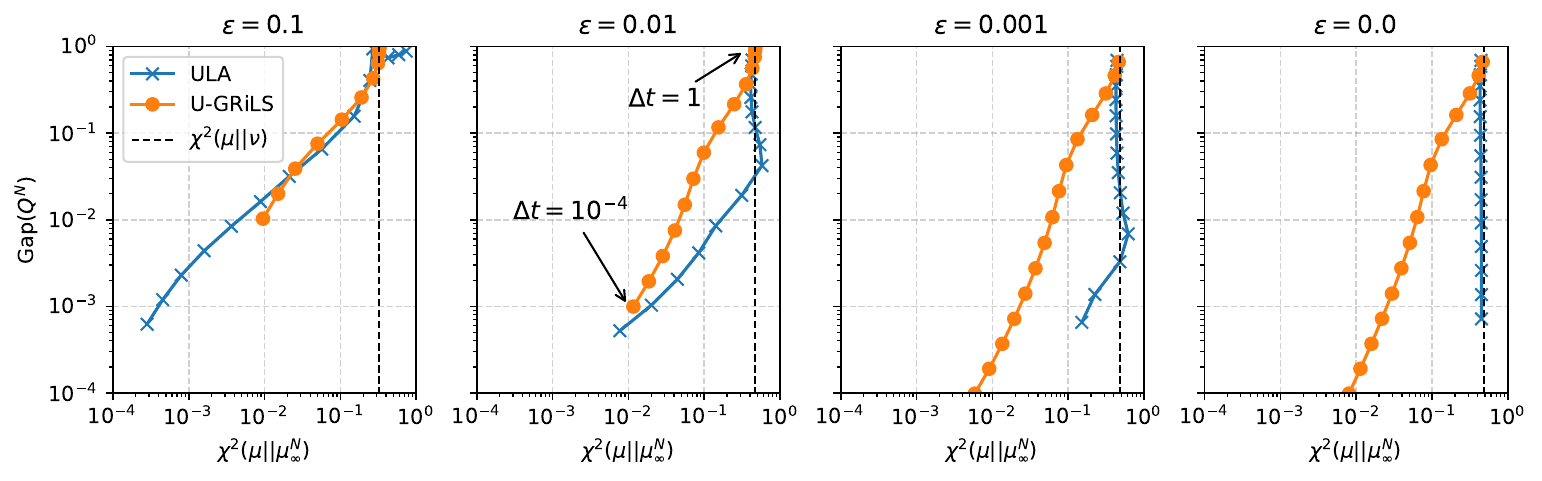}
    \caption{Unadjusted algorithms: spectral gap V.S. bias computed for different values of $\Delta t \in[10^{-4},10^0]$. Top row: Gaussian mixture $\mu_h$ for $h\in\{1,2,4,8\}$.
     Bottom row: smoothed piecewise constant density $\mu_\varepsilon$ for $\varepsilon\in\{10^{-1},10^{-2},10^{-3},0\}$. }
    \label{fig:C_vs_bias}
\end{figure}

Note that the worth performance of ULA is for the smoothed piecewise constant density $\mu_{\varepsilon=0}$ for which the bias do not decrease. This is not surprinsing, since $\nabla \ln \mu_{(\varepsilon=0)}(x) = 0$ almost surely. In contrast, the stationary measure of U-GRiLS is consistently approaching $\mu_{(\varepsilon=0)}$ as $\Delta t \rightarrow 0$, as shown on Figure \ref{fig:StepMixture_Convergence_UGRiLS}.

\begin{figure}
    \centering
    \begin{subfigure}[b]{0.55\textwidth}
        \includegraphics[width=\textwidth]{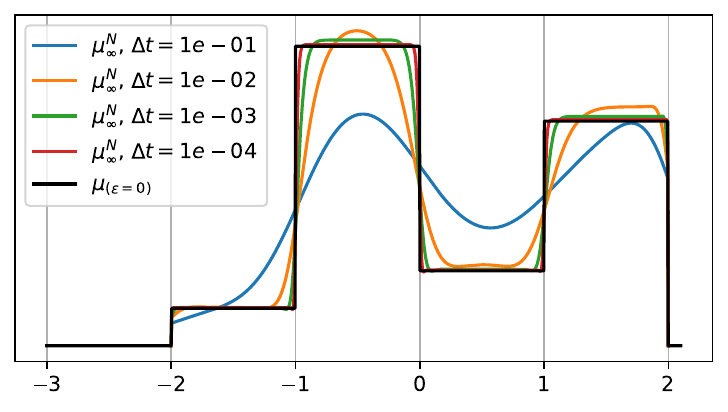}
        \caption{Stationary measures $\mu_\infty^N$}
    \end{subfigure}
    \hfill
    \begin{subfigure}[b]{0.4\textwidth}
        \centering
        \begin{minipage}[c][5cm][c]{\linewidth} %
            \centering \small
            \begin{tabular}{|c|c|}
              \hline
              $\Delta t$ & $\chi^2(\mu_{\varepsilon=0}||\mu_\infty^N)$ \\ \hline
              $10^{-1}$ & $1.44 \times 10^{-1}$ \\
              $10^{-2}$ & $5.21 \times 10^{-2}$ \\
              $10^{-3}$ & $1.99 \times 10^{-2}$ \\
              $10^{-4}$ & $6.55 \times 10^{-3}$ \\ \hline
          \end{tabular}
        \end{minipage}
        \caption{Bias}
    \end{subfigure}
    \caption{U-GRiLS targeting the piecewise constant density $\mu_{(\varepsilon=0)}$: stationary measures $\mu_\infty^N$ and bias $\chi^2(\mu||\mu_\infty^N)$ for different $\Delta t\in[10^{-1},10^{-4}]$.}
    \label{fig:StepMixture_Convergence_UGRiLS}
\end{figure}

\subsection{Metropolis Adjusted Algorithms}

Figure \ref{fig:C_vs_alpha} shows the performance of the Metropolis adjusted algorithms by plotting $\mathrm{Gap}(P^N)$ as a function of the rejection rate $1-\E[\alpha^N]$.
Both quantities are being computed for a range of $\Delta t\in[10^{-4},10^0]$.
For all considered algorithms, the rejection rate is known to behave like $1-\E[\alpha^N] = \mathcal{O}(\sqrt{\Delta t})$ with the expection of MALA for which $1-\E[\alpha^N] = \mathcal{O}(\Delta t^{3/2})$, see \emph{e.g.} \cite{roberts1998optimal}.
Thus we expect
$$
 \mathrm{Gap}(Q^N) ~=~
 \begin{cases}
  \mathcal{O}\left( (1-\E[\alpha^N] )^{2/3}\right) & \text{for MALA,} \\
  \mathcal{O}\left( (1-\E[\alpha^N] )^2\right)  &\text{for other algorithms.}
 \end{cases}
$$
This asymptotic behaviors are clearly observed for $\mu_h$ for $h\in\{1,2,4\}$.
Is it worth to note that, albeit being based on the discretization of a Langevin dynamic, GRiLS do not heritate from the favourable rejection rate $1-\E[\alpha^N] = \mathcal{O}(\Delta t^{3/2})$ and has a spectral gap behaving as $\mathrm{Gap}(Q^N)\sim ( 1-\E[\alpha^N] )^2$.

Among all algorithms, the independence sampler (IS) consistently yields the smallest spectral gap, making it the most effective Metropolis-adjusted algorithm for the one-dimensional problems we consider here.
This indicates that, for the Gaussian mixture $\mu_h$ (top row), the most efficient mechanism to jump across the modes is to draw a independent sample.
Note also that the expected rejection rate $1-\E[\alpha^N]$ of IS increases rapidly with $h$, corroborating formula \eqref{eq:Alpha_IS}. This means that, albeit having the largest spectral gap, IS accepts less frequently when $h\gg1$.

We also observe that GRiLS consistently performs better than pCN and, for large $\Delta t$, both GRiLS and pCN recover IS, which is consistent with their construction (remember Remark \ref{rmk:IS}).
It is worth noting that, for $h=8$, GRiLS achieves significantly largest spectral gap over a wider range of small acceptance probability.
This indicates that GRiLS is able to explore efficiently each modes (small rejection rate) while still being able to jump across modes (large spectral gap).

When considering MALA, we observe on both target densities that its performance deteriorates significantly for large $h$ and small $\varepsilon$. Note also that, for the piecewise constant density $\mu_{\varepsilon = 0}$, MALA reduces to RW.

\begin{figure}
    \centering
    \includegraphics[width=\textwidth]{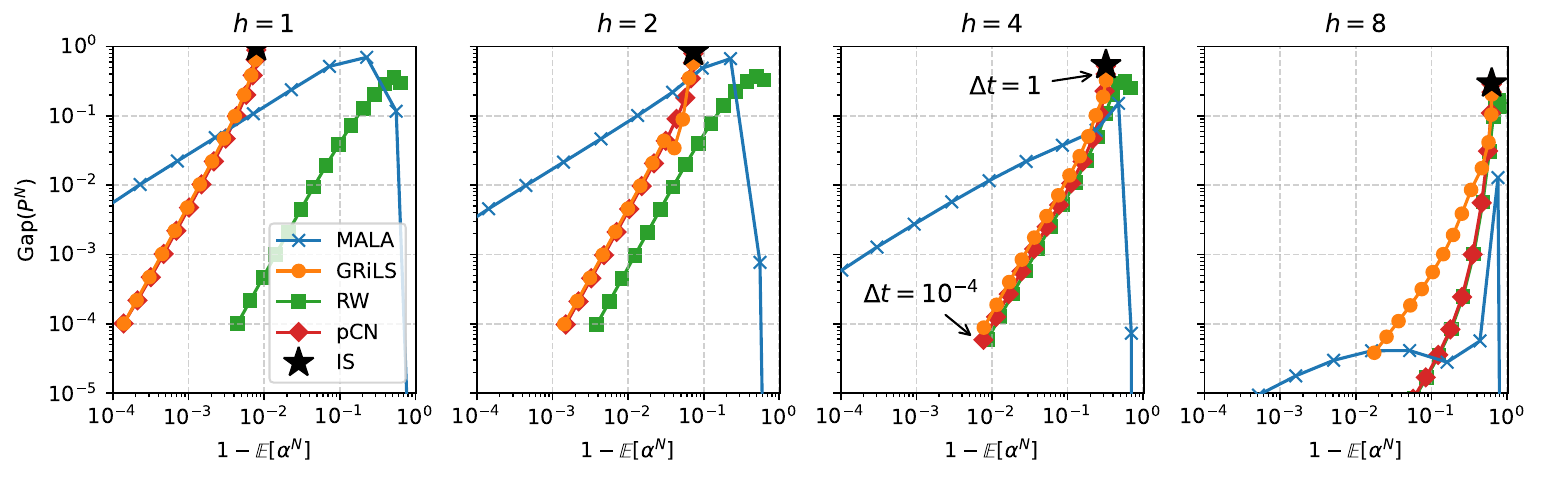}
    \includegraphics[width=\textwidth]{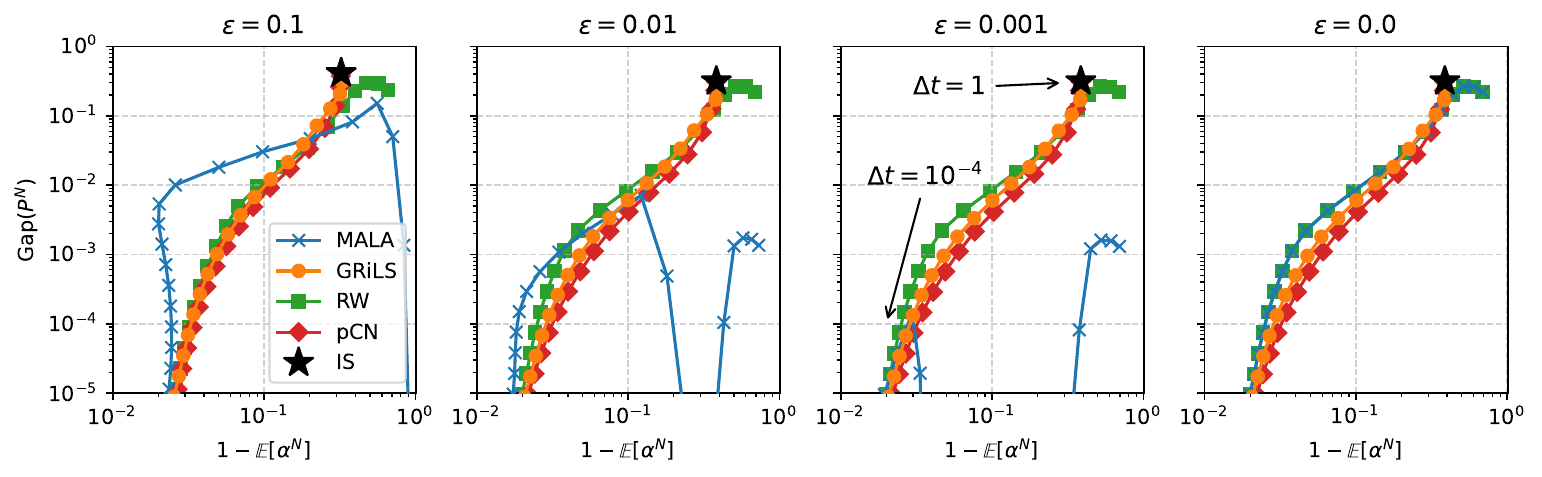}
    \caption{Adjusted algorithms: spectral gap V.S. expected rejection rate for different values of $\Delta t \in[10^{-4},10^0]$. Top row: Gaussian mixture $\mu_h$ for $h\in\{1,2,4,8\}$.
     Bottom row: smoothed piecewise constant density $\mu_\varepsilon$ for $\varepsilon\in\{10^{-1},10^{-2},10^{-3},0\}$. }
    \label{fig:C_vs_alpha}
\end{figure}

\section{Higher-Dimensional Numerical Experiments} \label{sec:HigherDimensional}

In this section, we demonstrate the performance of BE-GRiLS on multimodal densities\footnote{The implementation to reproduce the numerical results can be found at:~\url{https://github.com/baptistar/GradientFreeRiemannianLangevinSampler}.}. Subsection~\ref{sec:benchmark} considers two-dimensional benchmark densities, and  Subsection~\ref{sec:increasing_dimension} exhibits the performance with a ten-dimensional mixture density. %

\subsection{Benchmark Two-Dimensional Distributions} \label{sec:benchmark}
First, we assess the performance of BE-GRiLS on  benchmark non-Gaussian densities $\mu \in \mathcal{P}(\R^2)$ including a mixture of Gaussians, the two moons and the two rings densities defined by:
\begin{align*}
\mu(x) &\propto \frac{1}{3} \sum_{k=1}^3 \exp \left(-\frac{\|x -m_k\|^2}{\sigma^2} \right)\\
\mu(x) &\propto (1 + \exp^{4x_1/a}) \exp \left(-\frac{(\|x\| - 1)^2}{b} - \frac{(x_1 - 2)^2}{2a} \right) \\
\mu(x) &\propto \frac{1}{\|x\|} \left(-\frac{(\|x\| - 1)^2}{2\sigma_R^2} - \frac{(\|x\| - 2)^2}{2\sigma_R^2}\right), 
\end{align*}
respectively. For the mixture of Gaussians example, we set  $m_1 = [-\sqrt{3}, 1]$, $m_2 = [\sqrt{3}, 1]$ and $m_3 = [0, 2]$ so that the modes are centered at the vertices of an equilateral triangle, %
and $\sigma^2 = 0.1$. For the two moons example, we set $a = 0.08, b = 0.08$. For the two rings example, we set $\sigma_R = 0.1$.

We evaluate the block ensemble GRiLS in comparison to two gradient-free methods: adaptive Metropolis~\cite{haario2001adaptive}, which updates the proposal covariance adaptively based on the history of the Markov chain, and the Affine Independent Ensemble Sampler (AIES, \cite{goodman2010ensemble}) which, similarly to {BE-GRiLS}, is based on an ensemble of particles.
We also compare to the gold-standard Metropolis Adjusted Langevin Algorithm (MALA), which %
has access to gradients from the target density and whose covariance is also adapted based on the history of the chain. For AIES and BE-GRiLS, %
we initialize the samplers using $N=10$ particles drawn from a standard Gaussian density, while the Adaptive Metropolis and Adaptive MALA algorithms are initialized at the MAP of the target density with an identity covariance matrix for the Gaussian proposal. We run BE-GRiLS with $P=2$ blocks of equal size. Each method is run for 21,000 steps and we consider 1,000 steps as burn-in. We select $\Delta t = 0.3$ in these experiments to minimize the probability of missing modes due to initialization of the chain.

Figures~\ref{fig:mog}-\ref{fig:two_rings} plot the sample history relative to the contours of each target density and the trace plots for the two coordinates. Overall, we observe that only BE-GRiLS consistently captures all modes of the target density. Moreover, the trace plots demonstrate fast mixing between the modes as a result of the preconditioning. We note that while adaptive Metropolis can adapt the covariance structure to capture some of these two-dimensional densities with a sufficiently large step-size, the next subsection will demonstrate that it can generally miss modes in higher-dimensional problems.

\begin{figure}
    \centering
    \includegraphics[width=\textwidth]{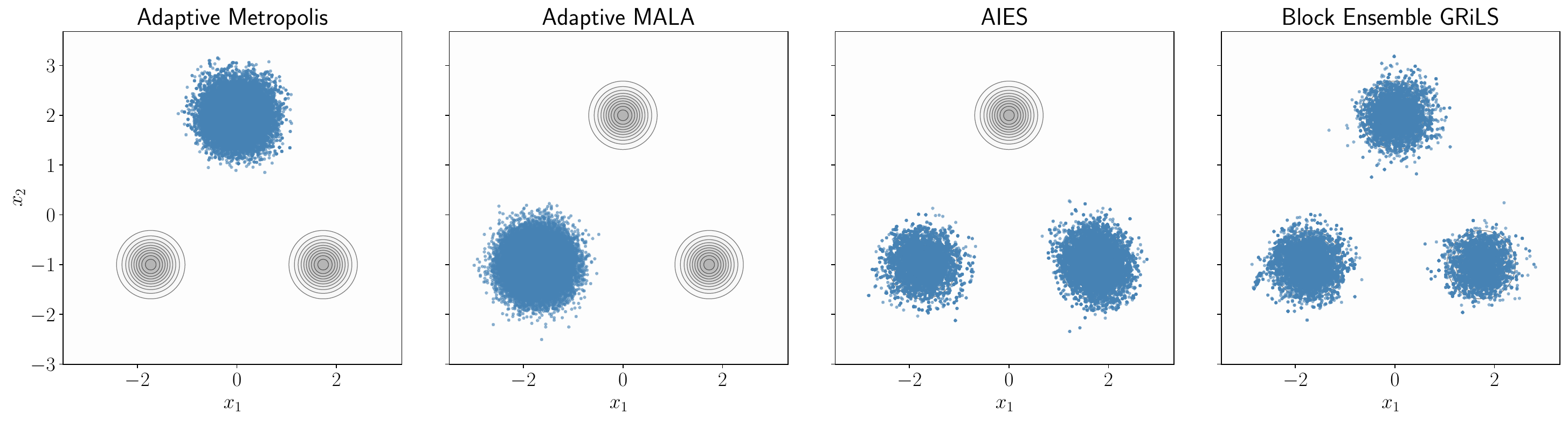}
    \includegraphics[width=\textwidth]{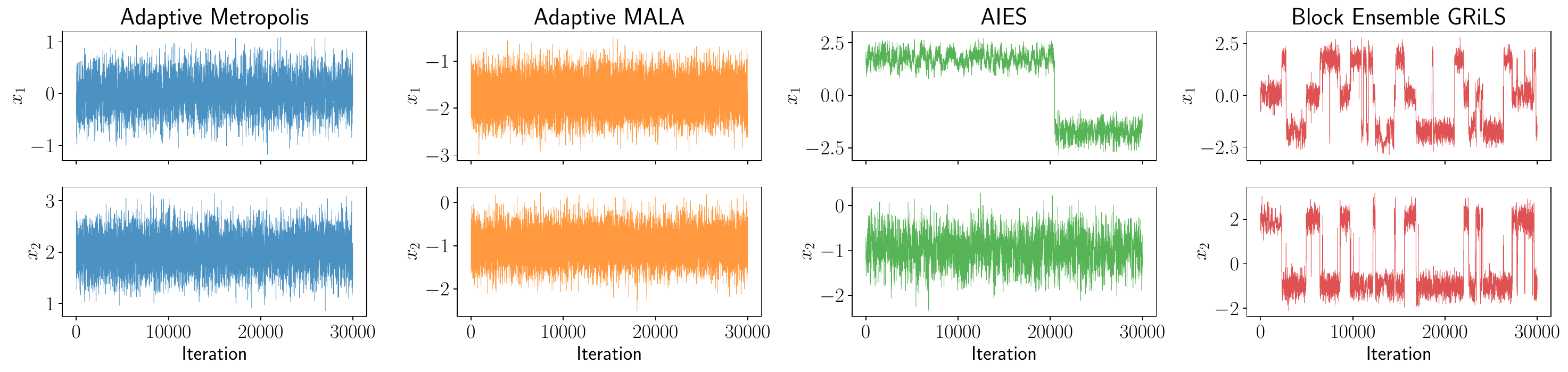}
    \caption{Samples and trace plots for the mixture of Gaussians density across four algorithms. The BE-GRiLS algorithm shows fast mixing behavior across and coverage of all three modes.~\label{fig:mog}}
\end{figure}

\begin{figure}[!ht]
    \centering
    \includegraphics[width=\textwidth]{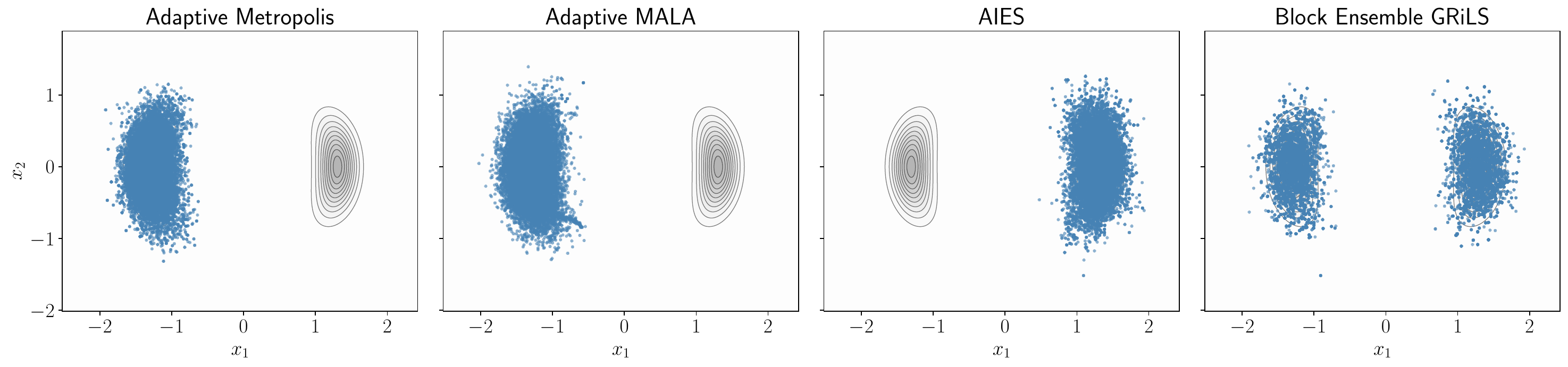}
    \includegraphics[width=\textwidth]{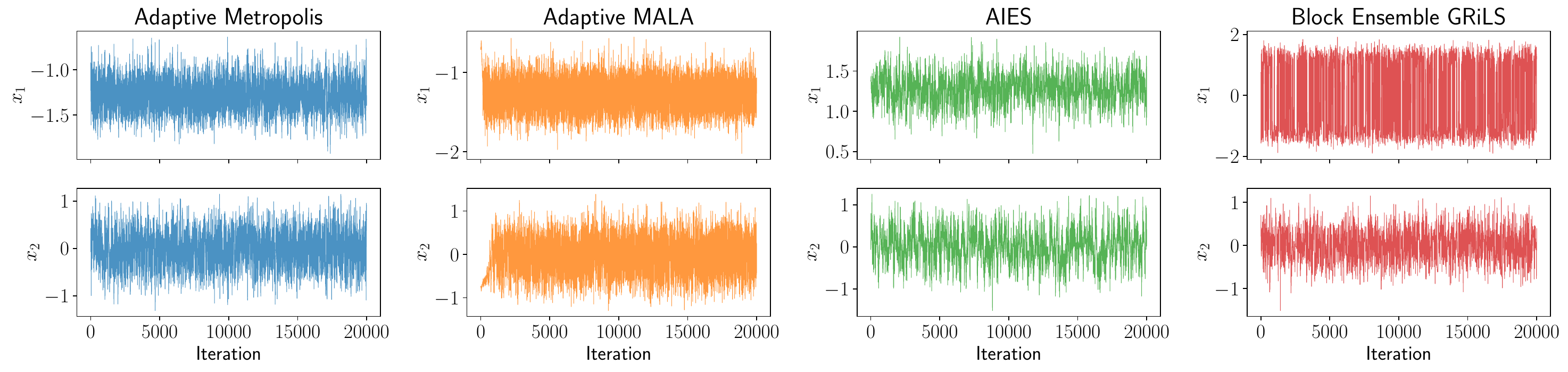}
    \caption{Samples and trace plots for the two moons density across four algorithms.\label{fig:dualmoons}}
\end{figure}

\begin{figure}[!ht]
    \centering
    \includegraphics[width=\textwidth]{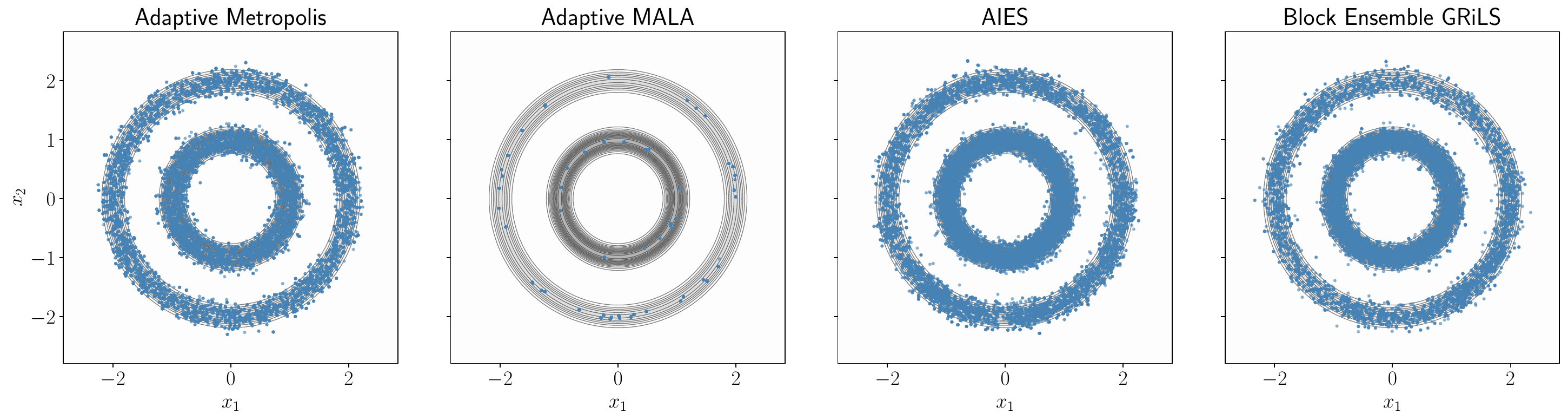}
    \includegraphics[width=\textwidth]{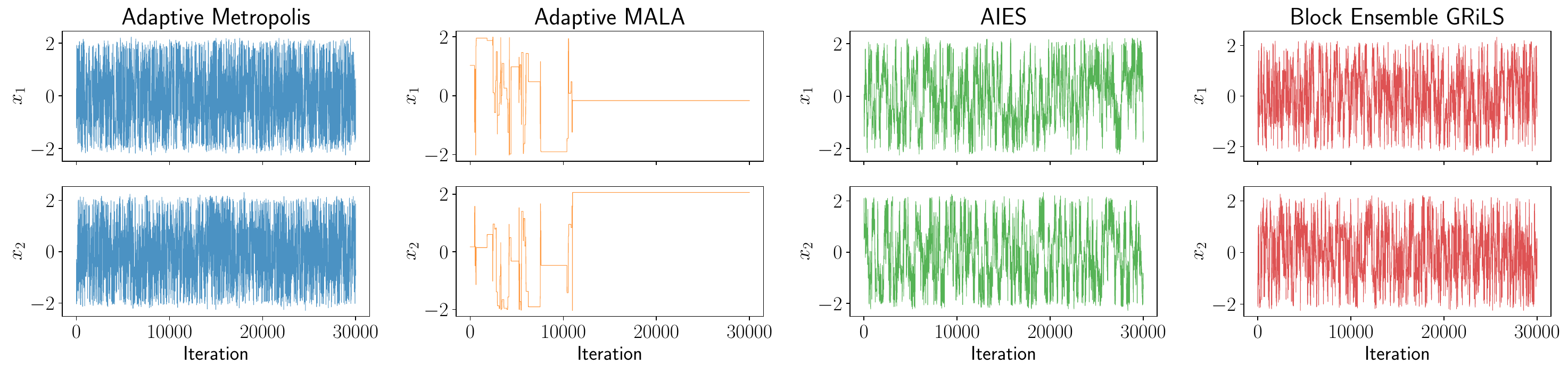}
    \caption{Samples and trace plots for the two rings density across four algorithms.\label{fig:two_rings}}
\end{figure}

\subsection{Mixture of Gaussians in Ten Dimension} \label{sec:increasing_dimension}

Lastly, we assess the performance of BE-GRiLS on approximating the first two moments of a $d=10$ dimensional Gaussian mixture target density with three equal-weight components. The first two coordinates match the target density of the Gaussian mixture in Subsection~\ref{sec:benchmark}, while the
remaining $d-2$ coordinates follow the law of independent standard normal variables. %
We compare the algorithm's performance to adaptive random-walk Metropolis (AM), MALA with an adaptive covariance, and the affine-invariant ensemble sampler (AIES).

For each algorithm we run $K = 10^5$ total iterations with the first $10^3$ steps discarded as burn-in. We initialize the single-chain methods (AM and MALA) from the target's MAP and the ensemble methods (AIES and BE-GRiLS) from $N = 40$ walkers drawn from
$\mathcal N(\mu_{\rm MAP}, \Sigma_\pi)$, where $\Sigma_\pi$ is the true global covariance in order to maximize the success of all methods. The step sizes $\Delta t$ (or the AIES stretch scale factor) are independently tuned for each algorithm by a pilot run of length $10^3$ steps by maximizing the expected squared jump distance per dimension,
$\mathrm{ESJD}/d = \tfrac1d\,\mathbb E\|x_{t+1}-x_t\|_2^2$. This criterion avoids select an overly small step, which would miss certain modes of the multimodal density~\cite{pasarica2010adaptively}. %

\begin{figure}[!ht]
    \centering
    \includegraphics[width=\textwidth]{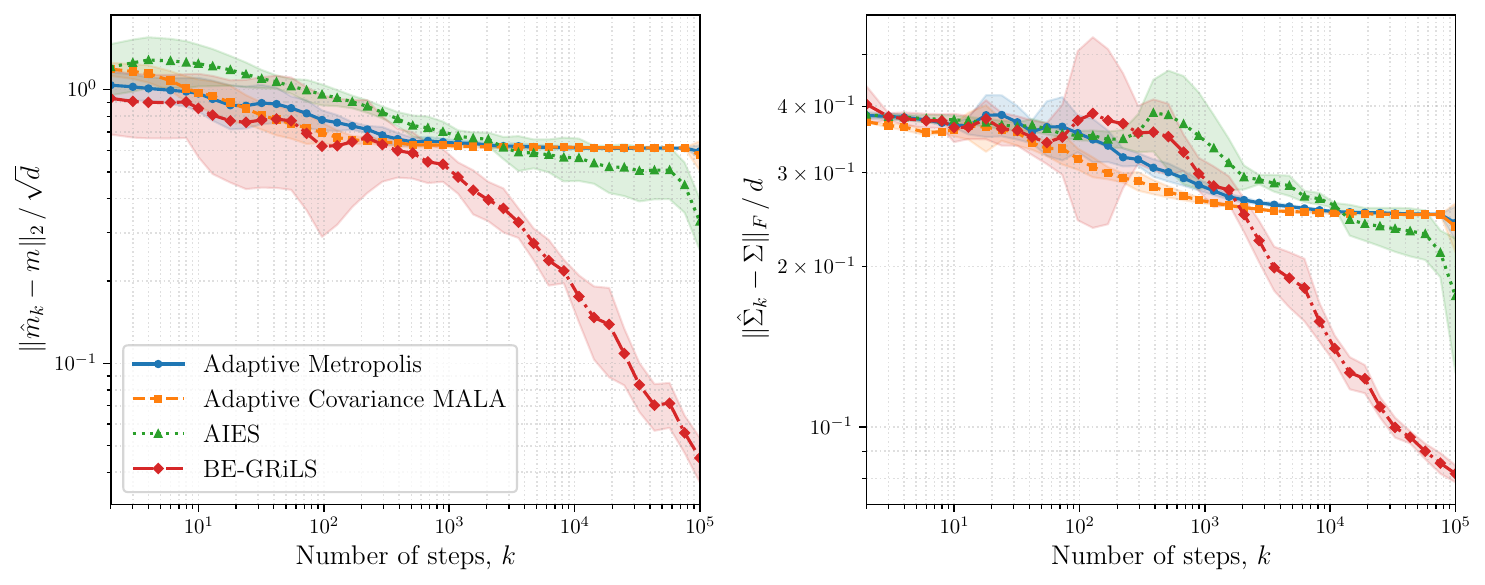}
    \caption{Convergence of the empirical mean (left) and covariance (right) for the $d=10$ dimensonal Gaussian mixture across four algorithms.\label{fig:converence_moments}}
\end{figure}

For each iteration $k$, we compute the errors between the running empirical mean and covariance $(\hat{m}_k,\hat\Sigma_k)$ and the true mean and covariance $(m,\Sigma)$ of the target density. The errors normalized by an appropriate dimension scaling are given by:
\[
\frac{\lVert \hat{m}_k - m\rVert_2}{\sqrt d}, \qquad
\frac{\lVert \hat \Sigma_k - \Sigma\rVert_F}{d}.
\]
Figure~\ref{fig:converence_moments} plots the mean and covariance error as a function of increasing iterations for the four algorithms. The results are averaged
over $5$ independent repetitions with the mean reported along with one standard deviation in the figure. %
Across the tested settings, BE-GRiLS generally attains the lowest normalized mean and covariance errors at a given sample budget than the single-chain AM and MALA baselines or the ensemble method AIES. As in Figure~\ref{fig:mog}, we observe fast mixing of BE-GRiLS across the three modes of the density, unlike the alternative methods which are more likely to become trapped within a single mixture component and hence result in lower global accuracy.

\section{Discussions and Conclusions}

In this work, we have introduced GRiLS, a gradient-free sampler based on the time discretization of a Riemannian Langevin dynamics, and demonstrated its connections to several established MCMC methods. Specifically, GRiLS unifies a broad family of proposals: it recovers an independence sampler (IS) in the limit $\Delta t\rightarrow\infty$, it simplifies to preconditioned Crank-Nicolson (pCN) scheme when the parameter $\theta_k=\theta$ is held constant, and shares the structural character of MALA/ULA through its Langevin foundation as $\Delta t\rightarrow 0$.
Moreover, the ensemble variant BE-GRiLS has a close resemblance to Consensus-Based Sampling (CBS), with each proposal combining a contraction toward the ensemble mean with a Gaussian perturbation. This unifying perspective suggests that GRiLS provides a principled and flexible framework that interpolates between and generalizes several state-of-the-art MCMC proposals.

Several directions for future work remain open. An important question in practice concerns the systematic selection of the step size $\Delta t$: as in gradient-based optimization, this choice is critical for algorithmic performance, and principled heuristics or adaptive strategies remain to be developed. On the theoretical side, it is natural to ask whether the classical optimal scaling analysis for MALA as in \cite{roberts1998optimal} can be adapted to GRiLS, which, despite being gradient-free, is fundamentally a Langevin sampler. Also, following \cite{liu2026spectral,andrieu2024explicit}, a convergence analysis via Cheeger inequalities, offers another promising theoretical avenue. Beyond the Gaussian setting, generalizing GRiLS to non-Gaussian reference measures $\nu$ would broaden its applicability, as outlined in Remark \ref{rmk:nonGaussianReference}. Improving the accuracy of the time integration through higher-order quadrature schemes represents a further refinement of the method.
In addition, as done in CBS \cite{carrillo2022consensus}, a rigorous analysis of the mean-field limit of BE-GRiLS as $N\rightarrow\infty$ would provide a deeper understanding of the ensemble dynamics in a non-stationary regime, that is, when the target mean and covariance are poorly estimated by the ensemble.

Finally, a critical question that this paper does not fully resolve concerns the robustness of GRiLS in high dimensions $d\gg 1$. As shown in Equation \eqref{eq:GRiLS_AcceptanceRate}, the acceptance rate of the Metropolis-corrected GRiLS algorithm degenerates as the dimension grows, revealing an intrinsic limitation of proposals with location-dependent covariance. A promising remedy, to be pursued in future work, is to confine the spatial variations of the proposal covariance to a low-dimensional subspace, in the spirit of \cite{lelievre2025improving}, thereby mitigating this adverse effect of dimensionality.

\paragraph{Acknowledgments.} The authors appreciate discussion of some of the results with Björn Sprungk and Simon Barthelm\'{e} and Bryan Toto.

\appendix

\section{Proof for \eqref{eq:GaussianMixture_1d} and \eqref{eq:GaussianMixture_dd}}\label{proof:GaussianMixture}

Inequality \eqref{eq:GaussianMixture_1d} can be deduce from \eqref{eq:GaussianMixture_dd} by letting $d=1$, $N=2$, $m_i=\pm h$. To proof \eqref{eq:GaussianMixture_dd}, we use the following proposition.

\begin{proposition}\label{prop:GaussianMixture}
 Let $\mu(x) = \sum_{i=1}^N \omega_i \mu_i(x)$ be a $d$-dimensional Gaussian mixture, where $\mu_i = \mathcal{N}(m_i,\Sigma_i)$, $\Sigma_i\succ0$, for some weights $\omega_i\geq0$ that sum to one. Its mean and covariance are given by
 $$
  m=\sum_{i=1}^N \omega_i m_i , \quad\text{and}\quad \Cov_\mu = \sum_{i=1}^N \omega_i \left( \Sigma_i + (m_i-m)(m_i-m)^\top \right) .
 $$
 For any matrix $\Sigma$ such that $\Sigma \succ \Sigma_i$ for all $i\leq N$, the Gaussian measure $\nu=\mathcal{N}(m,\Sigma)$ satisfies
\begin{equation}\label{eq:GaussianMixture}
 \sup_{x\in\R^d}\frac{\mu(x)}{\nu(x)}
 \leq \sum_{i=1}^N \omega_i \sqrt{\frac{|\Sigma|}{ |\Sigma_i|}} \exp\left(\frac{\|m_i-m\|_{(\Sigma-\Sigma_i)^{-1}}^2}{2} \right) .
\end{equation}
\end{proposition}

\begin{proof}
Let $A_i=\Sigma_i^{-1}-\Sigma^{-1}$ and $a_i = (\Sigma_i^{-1}m_i-\Sigma^{-1}m)$. Because $\Sigma\succ \Sigma_i$, we have $A_i \succ 0$ and then
\begin{align*}
 \frac{\mu(x)}{\nu(x)}
 &= \sum_{i=1}^N \omega_i\sqrt{\frac{|\Sigma|}{ |\Sigma_i|}} \exp\left(-\frac{\|x-m_i\|_{\Sigma_i^{-1}}^2}{2}+\frac{\|x-m\|_{\Sigma^{-1}}^2}{2}\right)\\
 &= \sum_{i=1}^N \omega_i\sqrt{\frac{|\Sigma|}{ |\Sigma_i|}} \exp\left(
 -\frac{x^\top( \Sigma_i^{-1}- \Sigma^{-1})x}{2}
 + x^\top( \Sigma_i^{-1}m_i - \Sigma^{-1} m  )
 -\frac{\|m_i\|_{\Sigma_i^{-1}}^2}{2}+\frac{\|m\|_{\Sigma^{-1}}^2}{2}
 \right)\\
 &= \sum_{i=1}^N \omega_i \sqrt{\frac{|\Sigma|}{ |\Sigma_i|}} \exp\left(-\frac{\|x-A_i^{-1}a_i\|_{A_i}^2}{2}+\frac{\|A^{-1}a_i\|_{A_i}^2}{2}-\frac{\|m_i\|_{\Sigma_i^{-1}}^2}{2} + \frac{\|m\|_{\Sigma^{-1}}^2}{2}\right)\\
 &= \sum_{i=1}^N \omega_i \sqrt{\frac{|\Sigma|}{ |\Sigma_i|}} \exp\left(-\frac{\|x-A_i^{-1}a_i\|_{A_i}^2}{2}+\frac{\|m_i-m\|_{(\Sigma-\Sigma_i)^{-1}}^2}{2}\right) .
\end{align*}
For the last step, we used
\begin{align*}
 \|A_i^{-1}a_i\|_{A_i}^2
 &= (\Sigma_i^{-1}m_i-\Sigma^{-1}m)^\top (\Sigma_i^{-1}-\Sigma^{-1})^{-1}(\Sigma_i^{-1}m_i-\Sigma^{-1}m)  \\
 &= m_i^\top (\Sigma_i-\Sigma_i\Sigma^{-1}\Sigma_i)^{-1}m_i -2 m_i^\top( \Sigma-\Sigma_i  )^{-1} m \\
 & + m^\top (\Sigma\Sigma_i^{-1}\Sigma-\Sigma)^{-1} m  \\
 &= \|m_i-m\|_{(\Sigma-\Sigma_i)^{-1}}^2 - m^\top (\Sigma-\Sigma_i)^{-1} m - m_i^\top (\Sigma-\Sigma_i)^{-1} m_i \\
 &+  m_i^\top (\Sigma_i-\Sigma_i\Sigma^{-1}\Sigma_i)^{-1}m_i   + m^\top (\Sigma\Sigma_i^{-1}\Sigma-\Sigma)^{-1} m \\
 &= \|m_i-m\|_{(\Sigma-\Sigma_i)^{-1}}^2 - m^\top \Sigma^{-1} m +  m_i^\top \Sigma_i^{-1}m_i ,
\end{align*}
where we employed the Woodbury formula
\begin{align*}
 (\Sigma_i-\Sigma_i\Sigma^{-1}\Sigma_i)^{-1} &= \Sigma_i^{-1} + (\Sigma-\Sigma_i)^{-1} \\
 (\Sigma\Sigma_i^{-1}\Sigma-\Sigma)^{-1} &= (\Sigma-\Sigma_i)^{-1} - \Sigma^{-1} .
\end{align*}
This yields
\begin{align*}
  \frac{\mu(x)}{\nu(x)}
  &\leq \sum_{i=1}^N \omega_i \sqrt{\frac{|\Sigma|}{ |\Sigma_i|}} \exp\left(\frac{\|m_i-m\|_{(\Sigma-\Sigma_i)^{-1}}^2}{2} \right) ,
\end{align*}
which is \eqref{eq:GaussianMixture}
\end{proof}

We now prove Inequality \eqref{eq:GaussianMixture_dd}. By Propositions \ref{prop:boundC} and \ref{prop:GaussianMixture}, the Gaussian mixture
\begin{equation}\label{eq:tmp2350789}
 \mu_h^N(x) = \frac{1}{N} \sum_{i=1}^N \mu_i(x), \qquad \mu_i = \mathcal{N}(m_i,I_d),
\end{equation}
satisfies
\begin{align*}
 C(\mu,W_{\mu_h^N})
 &\overset{\eqref{eq:boundC}}{\leq} \sup_{x\in\R^d}\frac{\mu(x)}{\nu(x)} \\
 &\overset{\eqref{eq:GaussianMixture}}{\leq} \sum_{i=1}^N \omega_i \sqrt{\frac{|\Sigma|}{ |\Sigma_i|}} \exp\left(\frac{\|m_i-m\|_{(\Sigma-\Sigma_i)^{-1}}^2}{2} \right) \\
 &\overset{\eqref{eq:tmp2350789}}{=} \frac{1}{N} \sum_{i=1}^N  \sqrt{\frac{|\Cov_\mu|}{ |I_d|}} \exp\left(\frac{\|m_i-m\|_{( \Cov_\mu-I_d )^{-1}}^2}{2} \right)
\end{align*}
Next, because $\|m_i-m\|\leq h$ for all $i$, we have $\Cov_\mu = I_d + \tfrac{1}{N} \sum_{i=1}^N (m_i-m)(m_i-m)^{\top} \preceq (1+h^2)I_d $ so that $|\Cov_\mu|\leq (1+h^2)^d$. Furthermore, we have
$$
\|m_i-m\|_{( \Cov_\mu-I_d )^{-1}}^2
= N \|m_i-m\|_{(  \sum_{i=1}^N (m_i-m)(m_i-m)^{\top} )^{-1}}^2
\leq N,
$$
for all $i\leq N$, so that the above inequality becomes
$
 C(\mu,W_{\mu_h^N})
 \leq (1+h^2)^{d/2} \exp(N/2)
$
which is \eqref{eq:GaussianMixture_dd}.

\section{Proof for inequality \eqref{eq:Cexplosion}}\label{proof:Cexplosion}

 Consider the Gaussian mixture $\mu = \frac{1}{2} \mathcal{N}(-h,1) +\frac{1}{2} \mathcal{N}(+h,1)$ in dimension $d=1$.
 In order to show that
 $$
  C(\mu, \Cov_\mu ) \geq \frac{ e^{h^2/3} }{1+h^2}  ,
 $$
 holds for any $h\geq 1/2$, one simply test the Poincar\'{e} inequality \eqref{eq:PI} with a smooth function $f$ that is approximately $\pm1$ on each of the two modes of $\mu$.
 Choose $f(x)= 2 \int_0^x \tfrac{ \exp(-t^2/2) }{\sqrt{2\pi}}\d t$. Because $f'(x)^2= 2/\pi \exp(-x^2)$ is even, we have
 \begin{align*}
  \E_\mu[f'^2]
  &= \int f'(x)^2 \frac{\exp(-(x-h)^2/2)}{\sqrt{2\pi}}\d x \\
  &= \frac{2}{\pi\sqrt{2\pi}} \int \exp\left( -x^2 - \frac{(x-h)^2}{2}  \right) \d x \\
  &= \frac{2}{\pi\sqrt{2\pi}} \int \exp\left( -\frac{(x-\frac{1}{3}h)^2}{2/3} +\frac{h^2}{6} -\frac{h^2}{2} \right) \d x\\
  &= \frac{2}{\pi\sqrt{3}} e^{-\frac{h^2}{3}} .
 \end{align*}
 Because $\Cov_\mu=1+h^2$ we deduce
 $$
  C(\mu, \Cov_\mu)
  \overset{\eqref{eq:PI}}{\geq} \frac{\Var_\mu(f)}{\E[ (f')^2 (1+h^2) ]}
  = \frac{\Var_\mu(f)\pi\sqrt{3} }{2} \frac{e^{\frac{h^2}{3}}}{ 1+h^2 } .
 $$
 Furthermore, since $f$ is odd we have $\E_\mu[f]=0$ and
 $$
  \Var_\mu(f)
  = \int f(x)^2 \frac{\exp(-(x-h)^2/2)}{\sqrt{2\pi}}\d x
  = \int f(x+h)^2 \frac{\exp(-x^2/2)}{\sqrt{2\pi}}\d x .
 $$
 Thus, the function $h\mapsto \Var_\mu(f)$ is monotonically increasing (because $f$ is) and one can check numerically that $\Var_\mu(f) > \frac{2}{\pi\sqrt{3}}$ as soon as $h\geq 1/2$. This gives the result.

{
\bibliographystyle{siam}
\bibliography{references}
}

\end{document}